%% file: main.tex
\documentclass{article}

\usepackage{microtype}
\usepackage{graphicx}
\usepackage{booktabs} 

\usepackage{hyperref}

\usepackage[accepted]{icml2020}

\usepackage{caption}
\usepackage{subcaption}
\usepackage{amssymb,amsthm,bbm,graphicx,placeins,enumitem}
\usepackage{amstext}
\usepackage{setspace,mathtools,thmtools,thm-restate}
\usepackage{float}
\usepackage{color,colortbl}
\usepackage{changepage}
\usepackage{booktabs}
\usepackage{array}
\usepackage{multirow}
\usepackage[font=small]{caption}
\usepackage{mathrsfs}  
\usepackage{thm-restate}
\usepackage{csquotes}
\usepackage[noblocks]{authblk}
\usepackage{titlesec}
\usepackage{relsize}
\usepackage{lipsum}  
\usepackage{chngcntr} 
\usepackage{longtable}
\usepackage{tikz}
\usetikzlibrary{automata,positioning,calc}
\usepackage{hyperref}
\allowdisplaybreaks

\usepackage{bbm}
\newcommand\myeq[1]{\stackrel{\textnormal{#1}}{=}}

\DeclareMathOperator*{\argmax}{arg\,max}

\newtheorem{thm}{Theorem}

\newtheorem{prop}{Proposition}
\newtheorem{lem}{Lemma}

\theoremstyle{definition}
\newtheorem{assumption}{Assumption}

\DeclareCaptionSubType[Alph]{figure}

\def\E{\mathbf{E}}
\def\P{\mathbf{P}}

\icmltitlerunning{Lookahead-Bounded Q-Learning}

\begin{document}
\allowdisplaybreaks
\twocolumn[
\icmltitle{Lookahead-Bounded Q-Learning}



\icmlsetsymbol{equal}{*}

\begin{icmlauthorlist}
\icmlauthor{Ibrahim El Shar}{pitt}
\icmlauthor{Daniel R.~Jiang}{pitt}
\end{icmlauthorlist}

\icmlaffiliation{pitt}{Department of Industrial Engineering, University of Pittsburgh, PA, USA}
\icmlcorrespondingauthor{Ibrahim El Shar}{ije8@pitt.edu}

\icmlkeywords{Machine Learning, ICML}

\vskip 0.3in
]



\printAffiliationsAndNotice{} 

\begin{abstract}
We introduce the lookahead-bounded Q-learning (LBQL) algorithm, a new, provably convergent variant of Q-learning that seeks to improve the performance of standard Q-learning in stochastic environments through the use of ``lookahead'' upper and lower bounds. To do this, LBQL employs previously collected experience and each iteration’s state-action values as dual feasible penalties to construct a sequence of sampled information relaxation problems. The solutions to these problems provide estimated upper and lower bounds on the optimal value, which we track via stochastic approximation. These quantities are then used to constrain the iterates to stay within the bounds at every iteration. 
Numerical experiments on benchmark problems show that LBQL exhibits faster convergence and more robustness to hyperparameters when compared to standard Q-learning and several related techniques.
Our approach is particularly appealing in problems that require expensive simulations or real-world interactions.
\end{abstract}
\vspace{-2em}
\section{Introduction}
\label{S:Intro}

Since its introduction by \citet{watkins1989learning}, Q-learning has become one of the most widely-used reinforcement learning (RL) algorithms \citep{sutton2018reinforcement}, due to its conceptual simplicity, ease of implementation, and convergence guarantees \citep{jaakkola1994convergence,tsitsiklis1994asynchronous,bertsekas1996neuro}. 
However, practical implementation of Q-learning on real-world problems is difficult due to the often high cost of obtaining data and experience from real environments, along with other issues such as overestimation bias \citep{szepesvari1998asymptotic,even2003learning,lee2019bias}.

In this paper, we address these challenges by focusing on a specific class of problems with \emph{partially known} models in the following sense: writing the transition dynamics as $s_{t+1} = f(s_t, a_t, w_{t+1})$, where $s_t$ and $a_t$ are the current state and action, $s_{t+1}$ is the next state, and $w_{t+1}$ is random noise, we consider problems where $f$ is known but $w_{t+1}$ can only be observed through interactions with the environment. This type of model is the norm in the control \citep{bertsekas1996neuro} and operations research \citep{powell2007approximate} communities. We propose and analyze a new RL algorithm called \emph{lookahead-bounded Q-learning} (LBQL), which exploits knowledge of $f$ to improve the efficiency of Q-learning and address overestimation bias. It does so by making better use of the observed data through estimating upper and lower bounds using a technique called \emph{information relaxation} (IR) \citep{brown2010information}.

Indeed, there are abundant real-world examples that fall into this subclass of problems, as we now illustrate with a few examples. In \textit{inventory control}, the transition from one inventory state to the next is a well-specified function $f$ given knowledge of a stochastic demand $w_{t+1}$ \citep{kunnumkal2008using}. For \textit{vehicle routing}, $f$ is often simply the routing decision itself, while $w_{t+1}$ are exogenous demands that require observation \citep{secomandi2001rollout}. In \textit{energy operations}, a typical setting is to optimize storage that behaves through linear transitions $f$ together with unpredictable renewable supply, $w_{t+1}$ \citep{kim2011optimal}.
In \textit{portfolio optimization}, $f$ is the next portfolio, and $w_{t+1}$ represents random prices \cite{rogers2002monte}.
In Section \ref{S:NE} of this paper, we discuss in detail another application domain that follows this paradigm: \emph{repositioning and spatial dynamic pricing for car sharing} \citep{he2019robust, bimpikis2019spatial}.

Although we specialize to problems with partially known transition dynamics, this should not be considered restrictive: in fact, our proposed algorithm can be integrated with the framework of model-based RL to handle the standard model-free RL setting, where $f$ is constantly being learned. We leave this extension to future work.

\textbf{Main Contributions.} We make the following methodological and empirical contributions in this paper.
\vspace{-1em}
\begin{enumerate}[itemsep=1pt,parsep=2pt]
    \item We propose a novel algorithm that takes advantage of IR theory and Q-learning to generate upper and lower bounds on the optimal value. This allows our algorithm to mitigate the effects of maximization bias, while making better use of the collected experience and the transition function $f$. A variant of the algorithm based on experience replay is also given.
    \item We prove that our method converges almost surely to the optimal action-value function. The proof requires a careful analysis of several interrelated stochastic processes (upper bounds, lower bounds, and the Q-factors themselves).
    \item Numerical experiments on five test problems show superior empirical performance of LBQL compared to Q-learning and other widely-used variants. Interestingly, sensitivity analysis shows that LBQL is more robust to learning rate and exploration parameters.
\end{enumerate}

\vspace{-1em}
The rest of the paper is organized as follows. In the next section, we review the related literature. In Section \ref{S:Back}, we introduce the notation and review the basic theory of IR. In Section \ref{S:Algorithm}, we present our algorithm along with its theoretical results. In Section \ref{S:NE}, we show the numerical results where LBQL is compared to other Q-learning variants. Finally, we state our conclusion and future work in Section \ref{S:Conclusion}.
\section{Related Literature}
\label{S:RL}
Upper and lower bounds on the optimal value have recently been used by optimism-based algorithms, e.g., \citet{dann2018policy} and \citet{zanette2019tighter}. 
These papers focus on finite horizon problems, while we consider the infinite horizon case. Their primary use of the lower and upper bounds is to achieve better exploration, while our work is focused on improving the action-value estimates by mitigating overestimation and enabling data re-use.

In the context of real-time dynamic programming (RTDP) \citep{barto1995learning}, Bounded RTDP \citep{mcmahan2005bounded}, Bayesian RTDP \citep{sanner2009bayesian} and Focused RTDP \citep{smith2006focused} propose extensions of RTDP where a lower bound heuristic and an upper bound are maintained on the value function. These papers largely use heuristic approaches to obtain bounds, while we use the principled idea of IR duality.

More closely related to our paper is the work of \citet{he2016learning}, which exploits multistep returns to construct bounds on the optimal action-value function, before utilizing constrained optimization to enforce those bounds. However, unlike our work, no theoretical guarantees are provided.
To the best of our knowledge, we provide the first asymptotic proof of convergence to the general approach of enforcing dynamically computed (noisy) bounds. 

There are also two papers that utilize IR bounds in the related setting of \emph{finite horizon} dynamic programming. \citet{jiang2017monte} use IR dual bounds in a tree search algorithm in order to ignore parts of the tree. 
Recent work by \citet{chen2019information} uses IR duality in a duality-based dynamic programming algorithm that converges monotonically to the optimal value function through a series of ``subsolutions'' under more restrictive assumptions (e.g., knowledge of probability distributions).
\section{Background}
\label{S:Back}
In this section, we first introduce some definitions and concepts from Markov decision process (MDP) theory. Then, we describe the basic theory of information relaxations and duality, which is the main tool used in our LBQL approach.
\vspace{-1em}
\subsection{MDP Model}
\label{S:MDP}
Consider a discounted, infinite horizon MDP with a finite state space $\mathcal S$, and a finite action space $ \mathcal A$, and a disturbance space $\mathcal W$. Let $\{w_t\}$ be a sequence of independent and identically distributed (i.i.d.) random variables defined on a probability space $(\Omega, \mathcal F, \mathbf{P})$, where each $w_t$ is supported on the set $\mathcal W$. Let $s_t \in \mathcal S$ be the state of the system at time $t$. We also define a state transition function $f: \mathcal S \times  \mathcal A \times W \rightarrow \mathcal S$, such that if action $a_t$ is taken at time $t$, then the next state is governed by $s_{t+1}= f(s_t, a_t, w_{t+1})$. This ``transition function'' model of the MDP is more convenient for our purposes, but we note that it can easily be converted to the standard model used in RL, where the transition probabilities, $p(s_{t+1} \,|\, s_t, a_t)$, are modeled directly. 
For simplicity and ease of notation, we assume that $w$ is independent\footnote{However, we can also allow for $(s,a)$-dependent $w$ with essentially no fundamental changes to our approach.} from $(s,a)$.  
Let $r(s_t, a_t)$ be the expected
reward when taking action $a_t \in \mathcal A$ in state $s_t \in \mathcal S$. We assume that the rewards $r(s_t,a_t)$ are uniformly bounded by $R_{\max}$ and for simplicity in notation, that the feasible action set $\mathcal A$ does not depend on the current state. As usual, a deterministic Markov policy $\pi \in \Pi$ is a mapping from states to actions, such that $a_t = \pi(s_t)$ whenever we are following policy $\pi$. We let $\Pi$ be the set of all possible policies (or the set of all ``admissible'' policies).

Given a discount factor $\gamma \in (0,1)$ and a policy $\pi \in \Pi$, the \emph{value} and the \emph{action-value} functions are denoted respectively by 
\vspace{-0.9em}
\begin{equation*}
    \begin{split}
        V^{\pi}(s)=\E\left[\sum_{t=0}^{\infty} \gamma^t r(s_t, a_t) \, \Bigl| \, \pi,\, s_0=s \right] \quad \text{and} \\ 
Q^{\pi}(s, a)=\E \left[\sum_{t=0}^{\infty} \gamma^t r(s_t, a_t) \, \Bigr| \, s_0=s, a_0=a, \pi\right],
    \end{split}
\end{equation*}
where the notation of ``conditioning on $\pi$'' refers to actions $a_t$ selected by $\pi(s_t)$.
The expectation $\E$, here and throughout the paper, is taken with respect to $\P$. Our objective is to find a policy $\pi \in \Pi$ such that starting from any state $s$, it achieves the optimal expected discounted cumulative reward. The value of this optimal policy $\pi^*$ for a state $s$ is called the optimal value function and is denoted by $V^*(s)= \max_{\pi} V^\pi(s)$. Specifically, it is well-known that an optimal policy selects actions according to $\pi^*(s) = \argmax_{a \in \mathcal A} Q^*(s,a)$,  where $Q^*(s,a)= \max_{\pi} Q^{\pi}(s,a)$ is the optimal action-value function  \citep{puterman2014markov}. The Bellman optimality equation gives the following recursion:
\vspace{-0.2em}
\[
    Q^*(s_t,a_t) =  r(s_t,a_t) + \gamma \, \E\Bigl[\max_{a_{t+1}} Q^*(s_{t+1},a_{t+1})\Bigr].
\]
The goal in many RL algorithms, including $Q$-learning \cite{watkins1989learning}, is to approximate $Q^*$.
\vspace{-0.5em}
\subsection{Information Relaxation Duality}
Now let us give a brief review of the theory behind information relaxation duality from \citet{brown2010information}, which is a way of computing an \emph{upper bound} on the value and action-value functions. This generalizes work by \citet{rogers2002monte}, \citet{haugh2004pricing}, and \citet{andersen2004primal} on pricing American options. Note that any feasible policy provides a lower bound on the optimal value, but computing an upper bound is less straightforward. The \emph{information relaxation} approach proposes to relax the ``non-anticipativity'' constraints on policies, i.e., it allows them to depend on realizations of future uncertainties when making decisions. Naturally, optimizing in the class of policies that can ``see the future'' provides an upper bound on the best admissible policy.
We focus on the special case of \emph{perfect information relaxation}, where full knowledge of the future uncertainties, i.e., the sample path $(w_1, w_2, \ldots)$, is used to create upper bounds. The naive version of the perfect information bound is simply given by
\vspace{-0.5em}
\[
V^*(s_0)  \le \E \left[ \max_{\mathbf a} \left\{\sum_{t=0}^{\infty} \gamma^t  r(s_t, a_t) \right\} \right],
\]
which, in essence, is an interchange of the expectation and max operators; the interpretation here is that an agent who is allowed to adjust her actions \emph{after} the uncertainties are realized achieves higher reward than an agent who acts sequentially. As one might expect, perfect information can provide upper bounds that are quite loose. 

The central idea of the information relaxation approach to strengthen these upper bounds is to simultaneously (1) allow the use of future information but (2) also penalize the agent for doing so by assessing a penalty on the reward function in each period. A penalty function is said to be \emph{dual feasible} if it does not penalize any admissible policy $\pi \in \Pi$ in expectation. Let $s_{t+1} = f(s_t, a_t, w_{t+1})$ be the next state, $\varphi : \mathcal S \times \mathcal A \rightarrow \mathbb R$ be a bounded function, and $w$ have the same distribution as $w_{t+1}$. Then, penalties involving terms of the form
\vspace{-0.1em}
\begin{equation}
\begin{split}
\textstyle z_t^\pi (s_t, a_t, w_{t+1} \, | \, \varphi) := \gamma^{t+1} \Bigl(\varphi(s_{t+1},\pi(s_{t+1}))\\ - \E \Bigl[ \varphi\bigl(f(s_t,a_t, w ),\pi(f(s_t,a_t, w ))\bigr) \Bigr]\Bigr)
\end{split}
\label{eq:normalpenalty}
\end{equation}
\vspace{-0.1em}
are dual feasible because
\vspace{-0.1em}
\[
\E \left[ \sum_{t=0}^{\infty} z_t^\pi(s_t,a_t, w_{t+1} \, | \, \varphi)\right]=0.\]
This is a variant of the types of penalties introduced in \citet{brown2017information}, extended to the case of action-value functions. Intuitively, if $\varphi$ is understood to be an estimate of the optimal action-value function $Q^*$ and $\pi$ an estimate of the optimal policy $\pi^*$, then $z_t^\pi$ can be thought of as the one-step value of future information (i.e., knowing $w_{t+1}$ versus taking an expectation over its distribution).

These terms, however, may have negative expected value for policies that violate non-anticipativity constraints. Let $\pi_{\varphi}$ be the policy that is greedy with respect to the bounded function $\varphi$ (considered to be an approximate action-value function).
Consider the problem constructed by subtracting the penalty term from the reward in each period and relaxing non-anticipativity constraints by interchanging maximization and expectation:
\vspace{-1.2em}
\begin{equation}
\begin{split}
    Q^{U}(s_0, a_0)&=\E \Bigl[ \max_{\mathbf a} \Bigl\{\sum_{t=0}^{\infty} \Bigl( \gamma^t   r(s_t, a_t) 
    \\  & - z_t^{\pi_\varphi}(s_t,a_t,w_{t+1} \, | \, \varphi) \Bigr)  \Bigr \} \Bigr], 
    \end{split}
    \label{E:Inf}
\end{equation}
where $\mathbf{a}  := (a_0, a_1, \ldots)$ is an infinite sequence of actions. \citet{brown2017information} shows that the objective value of this problem, $Q^{U}(s_0, a_0)$, is an upper bound on $Q^*(s_0,a_0)$. Our new approach to $Q$-learning will take advantage of this idea, with $\varphi$ and $\pi$ being continuously updated. Notice that in principle, it is possible to estimate the problem in \eqref{E:Inf} using Monte Carlo simulation. To do this, we generate infinitely long sample paths of the form $\mathbf{w}=(w_1, w_2, \ldots)$, and for each fixed $\mathbf{w}$, we solve the inner deterministic dynamic programming (DP) problem. Averaging over the results produces an estimate of the upper bound of $Q^{U}(s_0, a_0)$. 

\subsection{Absorption Time Formulation}
In practice, however, we cannot simulate infinitely long sample paths $\mathbf{w}$.
One solution is to use an \emph{equivalent} formulation with a finite, but random, horizon (see for e.g. Proposition 5.3.1 of \citealt{puterman2014markov}), where instead of discounting,
a new absorbing state $\tilde{s}$ with zero reward is added to the state space $\mathcal S$. This new state $\tilde{s}$ can be reached from every state and for any feasible action with probability $1-\gamma$.
We define a new state transition function $h$, which transitions to $\tilde{s}$ with probability $1-\gamma$ from every $(s,a)$, but conditioned on not absorbing (i.e., with probability $\gamma$), $h$ is identical to $f$. We refer to this as the \emph{absorption time formulation}, where
the horizon length $\tau := \min\{t: s_t=\tilde s \}$ has a geometric distribution with parameter $1-\gamma$ and the state transitions are governed by the state transition function $h$ instead of $f$.
Let $\mathcal Q$ be the set of bounded functions $\varphi$ such that $\varphi(\tilde{s}, a)=0$ for all $a \in \mathcal A$. The penalty terms for the absorption time formulation are defined in a similar way as (\ref{eq:normalpenalty}), except we now consider $\varphi \in \mathcal Q$:
\vspace{-0.5em}
\begin{equation}
\begin{split}
\textstyle \zeta_t^\pi &(s_t, a_t, w_{t+1} \, | \, \varphi) := \varphi(s_{t+1},\pi(s_{t+1})) \\ &- \E \Bigl[ \varphi \bigl(h(s_t,a_t, w ),\pi(h(s_t,a_t, w ))\bigr) \Bigr],
\end{split}
\label{E:za}
\end{equation}
where $s_{t+1}=h(s_t,a_t,w_{t+1})$. We now state a proposition that summarizes the information relaxation duality results, which is a slight variant of results in Proposition 2.2 of \mbox{\citet{brown2017information}}.
\begin{prop}[Duality Results, Proposition 2.2 in \citet{brown2017information}] The following duality results are stated for the absorption time formulation of the problem.
\vspace{-1em}
\begin{enumerate}[label={(\roman*)},itemindent=0.em]
    \item
    Weak Duality: For any $\pi \in \Pi$ and $\varphi \in \mathcal Q$,
    \vspace{-0.5em}
    \begin{equation}
    \begin{split}
    Q^{\pi}(s_0,a_0) &\le \E \Bigl[ \max_ {\mathbf{a}} 
    \sum_{t=0}^{\tau-1} \Big (r(s_t, a_t) \\ & \ \ - \zeta^{\pi_{\varphi}}_t(s_t,a_t, w_{t+1} \, | \, \varphi) \Big ) \Bigr]
    \end{split}
    \label{E:WD}
    \end{equation}
    \label{P:WD}
    \vspace{-2em}
    \item 
    Strong Duality: It holds that
    \vspace{-1em}
    \begin{equation} 
    \begin{split}
    Q^{*}(s_0,a_0) & = \inf_{\varphi \in \mathcal Q} \; \E \Bigl[ \max_{\mathbf a} \sum_{t=0}^{\tau-1}\Bigl( r(s_t,a_t) \\ & \ \ - \zeta^{\pi_\varphi}_t(s_t,a_t, w_{t+1}  \, | \, \varphi) \Bigr)
    \Bigr],
    \end{split}
    \label{E:SD}
    \end{equation} 
     with the infimum attained at $\varphi = Q^*$.
    \label{P:SD}
\end{enumerate}
\label{P:ATDR}
\end{prop}
\vspace{-0.9em}
The DP inside the expectation of the right hand side of \eqref{E:WD} is called the \emph{inner DP problem}. Weak duality tells us that by using a dual feasible penalty, we can get an estimated upper bound on the optimal action-value function $Q^*(s_0, a_0)$ by simulating multiple sample paths and averaging the optimal value of the resulting inner problems. Strong duality suggests that the  information gained from accessing the future is perfectly cancelled out by the optimal dual feasible penalty.

For a given sample path $\mathbf w=(w_1,w_2, \ldots, w_{\tau})$, each of the inner DP problems can be solved via the backward recursion
\vspace{-1.0em}
\begin{equation}
\begin{split}
 Q_t^{U}(s_t,a_t)   &=  r(s_t,a_t)   - \zeta_t^{\pi_\varphi}(s_t, a_t, w_{t+1} \, | \, \varphi) \\ & + \max_{a}  Q_{t+1}^{U}(s_{t+1},a),
 \end{split}
\label{E:QG}
\end{equation}
for $t=\tau-1, \tau-2, \ldots, 0$ with $s_{t+1}= h(s_t, a_t, w_{t+1})$ and $Q^{U}_{\tau}\equiv 0$ (as there is no additional reward after entering the absorbing state $\tilde{s}$). The optimal value of the inner problem is given by $Q_0^{U}(s_0,a_0)$.
\subsection{Lower Bounds using IR}
The penalty function approach also allows for using a feasible policy to estimate a \emph{lower bound} on the optimal value, such that when using a common sample path, this lower bound is guaranteed to be less than the corresponding estimated upper bound, a crucial aspect of our theoretical analysis.
Specifically, given a sample path $(w_1, w_2, \ldots, w_\tau)$, the inner problem used to evaluate a feasible policy $\pi \in \Pi$ is given by
\begin{equation}
\begin{split}
  Q^L_t(s_t,a_t) &=  r(s_t,a_t)- \zeta^{\pi}_t(s_t, a_t, w_{t+1} \, | \, \varphi) 
  \\ & + Q^L_{t+1}(s_{t+1},\pi(s_{t+1})),
\end{split}
\label{E:QL}
\end{equation}
for $t=0,\ldots,\tau-1,$ with $s_{t+1}= h(s_t, a_t, w_{t+1})$ and $Q^{ L}_{\tau}\equiv 0$. It follows that
$\E \bigl[ Q_0^L(s_0,a_0) \bigr]=Q^{\pi}(s_0,a_0)$,
as the penalty terms $ \zeta^\pi_t(s_t, a_t, w_{t+1} \, | \, \varphi)$ have zero mean. 

\section{QL with Lookahead Bounds}
 \label{S:Algorithm}

We now introduce our proposed approach, which integrates the machinery of IR duality with $Q$-learning in a unique way. An outline of the essential steps is given below.
\vspace{-1em}
\begin{enumerate}[itemsep=1pt,parsep=1pt]
\item On a given iteration, we first experience a realization of the exogenous information $w_{t+1}$ and make a standard $Q$-learning update.
\item We then set $\varphi$ to be the newly updated $Q$-iterate and compute noisy upper and lower bounds on the true $Q^*$, which are then tracked and averaged using a stochastic approximation step.
\item Finally, we project the $Q$-iterate so that it satisfies the averaged upper and lower bounds and return to Step 1.
\end{enumerate}
\vspace{-1em}

Figure \ref{F:LBQL_fig_tikz} shows an illustration of each of these steps at a given iteration of the algorithm. Since we are setting $\varphi$ to be the current $Q$-iterate at every iteration, the information relaxation bounds are computed using a \emph{dynamic sequence of penalty functions} and averaged together using stochastic approximation. The idea is that as our approximation of $Q^*$ improves, our upper and lower bounds also improve. As the upper and lower bounds improve, the projection step further improves the $Q$-iterates. It is this back-and-forth feedback between the two processes that has the potential to yield rapid convergence toward the optimal $Q^*$.
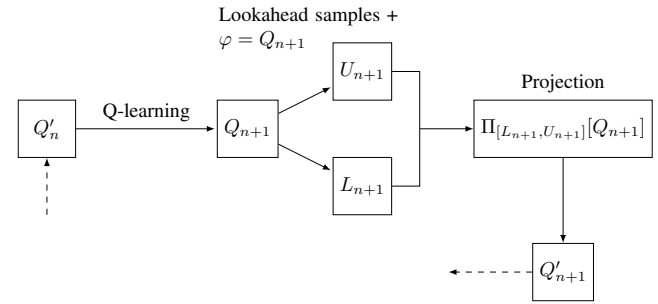
\begin{figure}[htb]
	\centering
	\resizebox{0.5\textwidth}{!}{
	\begin{tikzpicture}[ ->,>= latex, shorten >=1pt,node distance=1.8cm,on grid,auto]
	\node [draw,rectangle, minimum size=1cm] (Q'n)  {$Q'_n$};
	\node [draw,rectangle, minimum size=1cm] (Qn+1) at (3.5,0) {$Q_{n+1}$};
	\node [draw,rectangle, minimum size=1cm] (Un+1) at (5.5,1) {$U_{n+1}$};
	\node [draw,rectangle, minimum size=1cm] (Ln+1) at (5.5,-1) {$L_{n+1}$};
	\node [draw,rectangle, minimum size=1cm,label={Projection}] (proj) at (9,0) {$\Pi_{[L_{n+1},U_{n+1}]}[Q_{n+1}]$};
	\node [draw,rectangle, minimum size=1cm] (Q'n+1)  at (9,-2.5) {$Q'_{n+1}$};
	\node (inters) at (6.5,0) {};
	\node[text width=4cm,align=left](text) at (5,1.75) {Lookahead samples +\\ $\varphi=Q_{n+1}$};
	
	\path[->]
	(Q'n)
	edge node {Q-learning} (Qn+1)
	(Qn+1)
	edge node {} (Un+1)	
	(Qn+1)
	edge node {} (Ln+1);	
	\draw[shorten >=0pt,-,black] (Un+1) -| node[near start,above] {} (inters.center);
	\draw[shorten >=0pt,-,black] (Ln+1) -| node[near start,above] {} (inters.center);
	\draw[shorten >=0pt, ->,black] (inters.center) -- (proj);
	\draw[shorten >=0pt, ->,black] (proj) -- (Q'n+1);
	\draw[shorten >=0pt, ->,black]  (Q'n+1);
	\draw[dashed, shorten >=0pt, ->,black]  (Q'n+1) -- (7,-2.5);
	\draw[dashed, shorten >=0pt, ->,black]  (0,-1.5) --  (Q'n);
	\end{tikzpicture}
	}
	\caption{Illustration of LBQL Algorithm at iteration $n$.}
	\label{F:LBQL_fig_tikz}
\end{figure}

The primary drawback of our approach is that in the computation of the information relaxation dual bounds, expectations need to be computed. We first show an \emph{idealized} version of the algorithm where these expectations are estimated using unbiased samples of $w_{t+1}$ from a black-box simulator.
Later, we relax the need for a black-box simulator and show how our algorithm can be implemented with a replay-buffer. Both versions are analyzed theoretically and convergence results are provided.

\subsection{An Idealized Algorithm}
Let $\{w^1_{t+1}, w^2_{t+1}, \ldots, w^K_{t+1} \}$ be a \emph{batch} (as opposed to a sample path) of $K$ samples from the distribution of the exogenous information $w_{t+1}$ (i.e., from a black-box simulator). An empirical version of (\ref{E:za}) is simply given by:
\begin{equation}
\begin{split}
\textstyle \hat \zeta_t^\pi &(s_t, a_t, w_{t+1} \, | \, \varphi) :=\varphi(s_{t+1},\pi(s_{t+1})) \\ & - \frac{1}{K} \sum_{k=1}^K \varphi \bigl(h(s_t,a_t,w_{t+1}^k),\pi(h(s_t,a_t,w_{t+1}^k))\bigr),
\label{E:Empza}
\end{split}
\end{equation}
where $s_{t+1}=h(s_t,a_t,w_{t+1})$.
Given a sample path $\mathbf w=(w_1, w_2, \ldots,w_{\tau})$ of the absorption time formulation of the problem, analogues to (\ref{E:QG}) and (\ref{E:QL}) using $\hat{\zeta}_t^\pi$, where in \eqref{E:QL} we set $\pi=\pi_{\varphi}$ (i.e., the lower bound on the optimal value is constructed by approximately evaluating the feasible policy $\pi_\varphi$) are given by
\begin{align}
\begin{split}
 \hat Q_t^{U}(s_t,a_t)   =& \, r(s_t,a_t)  -\hat \zeta_t^{\pi_{\varphi}} (s_t, a_t, w_{t+1} \, | \, \varphi) \\ & + \max_{a}  \hat Q_{t+1}^{U}(s_{t+1},a) \label{E:EmpQG}
 \end{split}\\
 \begin{split}
 \hat Q^L_t(s_t,a_t) =& \, r(s_t,a_t)  -\hat \zeta_t^{\pi_{\varphi}} (s_t, a_t, w_{t+1} \, | \, \varphi) \\ & + \hat Q^L_{t+1}(s_{t+1},\pi_{\varphi}(s_{t+1}))
  \end{split}
\label{E:EmpQL}
\end{align}
for $t=0,1,\ldots,\tau-1,$ where $s_{t+1}= h(s_t, a_t, w_{t+1})$, $\hat Q^{ U}_{\tau}\equiv \hat Q^L_{\tau}\equiv0$, and  we assume that each call to $\hat{\zeta}_t^\pi$ uses a fresh batch of $K$ samples.
\begin{prop}
The valid upper and lower bound properties continue to hold in the empirical case:
\begin{equation*}
    \E[\hat Q^L_0(s, a)] \le Q^*(s, a) \le \E[\hat Q_0^{U}(s,a)  ],
\end{equation*}
\vspace{-0.5em}
for any state-action pair $(s,a)$.
\label{P:EmpDR}
\end{prop}

We include the proof in Appendix \ref{Appendix:ProofPropEmpDR}.
The proof is similar to that of Proposition 2.3(iv) of \citet{brown2010information}, except extended to the infinite horizon setting with the absorption time formulation.
A detailed description of the LBQL algorithm is given in Algorithm \ref{A:IV-LBQL}, where we use `$n$' for the iteration index in order to avoid confusion with the `$t$' used in the inner DP problems.
We use $\Pi_{[a,b]} [x]$ to denote $x$ projected onto $[a,b]$, i.e., $\Pi_{[a,b]} [x]=  \max \{\min \{x, b \}, a\}$, where either $a$ or $b$ could be $\infty$.
Let $ \rho = R_{\max}/(1-\gamma)$,
the initial lower and upper bounds estimates are set such that 
$ L_0(s,a) = -\rho $ and $U_0(s,a) = \rho$ for all $(s,a) \in \mathcal S  \times \mathcal A$.
The initial action-value $Q_0$ is set arbitrarily such that $L_0(s,a) \le Q_0(s,a) \le U_0(s,a)$ for all $(s,a) \in \mathcal S  \times \mathcal A$.

\begin{algorithm}[ht]
    \caption{Lookahead-Bounded Q-Learning}
    \label{A:IV-LBQL}
\begin{algorithmic}    
  \STATE {\bfseries Input:} Initial estimates $L_0 \le Q_0 \le U_0$, batch size $K$, and stepsize rules $\alpha_n(s,a)$, $\beta_n(s,a)$.
  \STATE {\bfseries Output:} Approximations $\{L_n\}$, $\{Q'_n\}$, and $\{U_n\}$.
  \STATE Set $Q_0' = Q_0$ and choose an initial state $s_0$.
  \FOR{$n = 0, 1, 2, \ldots$}
  \STATE Choose an action $a_n$ via some behavior policy (e.g., $\epsilon$-greedy). Observe $w_{n+1}$. Let
  \vspace{-0.5em}
  \begin{align}
\nonumber
  & Q_{n+1}(s_n,a_n) = Q'_n(s_n,a_n) + \alpha_n(s_n, a_n) \Bigl[ r_n(s_n,a_n) \\ & + \gamma \max_a Q'_n(s_{n+1},a) - Q'_n(s_{n},a_n) \Bigr ].
   \label{Qupdate}
\end{align}
  \vspace{-1.0em}
  \STATE Set $\varphi = Q_{n+1}$. Using one sample path $\mathbf{w}$, compute $\hat  Q^{U}_{0}(s_n,a_n)$ and $\hat Q^L_{0}(s_n,a_n)$ using  \eqref{E:EmpQG} and \eqref{E:EmpQL}.
  \STATE Update and enforce upper and lower bounds:
  \vspace{-0.5em}
\begin{align}
\nonumber
 &U_{n+1}(s_n,a_n) = \Pi_{[-\rho, \, \infty]}  \Bigl[ U_n(s_n,a_n)  \\ &
 + \beta_n(s_n,a_n) \, \bigl[ \hat Q^{U}_{0}(s_n,a_n) - U_n(s_n,a_n) \bigr]  \Bigr], \label{Uupdate} \\ 
\nonumber
  &L_{n+1}(s_n,a_n) = \Pi_{[\infty, \, \rho]} \Bigl[ L_n(s_n,a_n)   \\ & 
  + \beta_n(s_n,a_n) \, \bigl[\hat Q^L_{0}(s_n,a_n)- L_n(s_n,a_n) \bigr] \Bigr],  \label{Lupdate}\\
\nonumber
  &Q'_{n+1}(s_n,a_n) =  \\ & \Pi_{[ L_{n+1}(s_n,a_n),\, U_{n+1}(s_n,a_n)]}\left[Q_{n+1}(s_n,a_n)\right]. \label{Q'update}
\end{align}
  \ENDFOR
\vspace{-0.2em}
\end{algorithmic}
\end{algorithm}
\label{A:Idealized_algo_example}

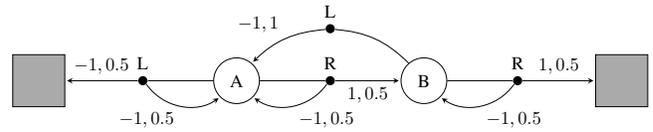
\begin{figure}[htb]
	\centering
	\resizebox{0.5\textwidth}{!}{
	\begin{tikzpicture}[ ->,>= stealth, shorten >=1pt,node distance=1.8cm,on grid,auto]
	\node [draw,rectangle,fill={rgb:black,1;white,2}, minimum size=1cm] (T1)  {};
	\node[circle,node distance=2cm,fill=black,inner sep=0pt,minimum size=5pt,label=above:{L}] (action_node_A_L) [ right=of T1] {};
	\node[state] (A) [right=of action_node_A_L] {A};
	\node[circle,fill=black,inner sep=0pt,minimum size=5pt,label=above:{R}] (action_node_A_R) [ right=of A] {};
	\node[node distance=1cm, circle,fill=black,inner sep=0pt,minimum size=5pt,label=above:{L}] (action_node_B_L) [ above=of action_node_A_R] {};
	\node[state] (B) [right=of action_node_A_R] {B};
	\node[circle,fill=black,inner sep=0pt,minimum size=5pt,label=above:{R}] (action_node_B_R) [ right=of B] {};
	\node [draw,node distance=2cm,rectangle,fill={rgb:black,1;white,2}, minimum size=1cm,] (T2) [right=of action_node_B_R] {};
	\path[->]
	(action_node_A_L)
	edge [swap] node {$-1, 0.5$} (T1)
	(action_node_A_L)
    edge [swap,bend right=45] node {$-1, 0.5$} (A)
	(A)
	edge [-,shorten >=0pt] node {} (action_node_A_L)
	(A)
	edge [-,swap,shorten >=0pt]  node {} (action_node_A_R)
	(action_node_A_R)
	edge [bend left=45]  node {$-1, 0.5$} (A)
	(action_node_A_R)
	edge [swap]  node {$1, 0.5$} (B)
	(B)
	edge [bend right=20,-, swap,shorten >=0pt]  node {} (action_node_B_L)
	(action_node_B_L)
	edge [bend right=20, swap]  node {$-1 ,1$} (A)
	(B)
	edge [-,shorten >=0pt] node {} (action_node_B_R)
	(action_node_B_R)  
	edge [] node {$1, 0.5$} (T2)
	(action_node_B_R)  
	edge [bend left=45] node {$-1, 0.5$} (B)
	;
	\end{tikzpicture}
	}
	\caption{A simple stochastic MDP.}
	\label{F:Example}
\end{figure}
\vspace{-0.5em}
\textbf{Example 1.} \textit{We demonstrate the idealized LBQL algorithm using the simple MDP shown in Figure \ref{F:Example}.  The MDP has two non-terminal states \texttt{A} and \texttt{B}. Each episode starts in state \texttt{A}, with a choice of two actions: right and left denoted by \texttt{R} and \texttt{L} respectively. 
The rewards and transition probabilities of taking an action in each state are shown on the edges in the figure.
Assume that the transitions are governed by the outcome of a fair coin. If the outcome is \texttt{Head} then we transition in the direction of our chosen action and in the opposite direction for a \texttt{Tail} outcome.
For a discount factor $\gamma=0.95$, the optimal policy is to go right at both \texttt{A} and \texttt{B}.
The optimal action-values are given by
$Q^*(A,R)=Q^*(B,R)=0$, $Q^*(A,L)=Q^*(B,L)=-1$. 
Consider applying the idealized version of LBQL described in Algorithm \ref{A:IV-LBQL}.} \newdimen\nodeSize
\nodeSize=4mm
\newdimen\nodeDist
\nodeDist=1.2cm

\def\angR{-70}
\def\angL{-110}
\def\rad{1}
\def\distxL{1}
\def\distyL{0.5}
\def\distSep{0.5}
\def\distBetwnStates{2.75}
\tikzset{
    position/.style args={#1:#2 from #3}{
        at=(#3.#1), anchor=#1+180, shift=(#1:#2)
    }
}

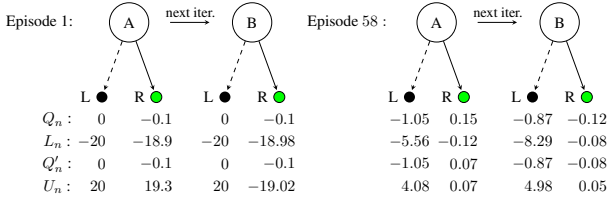
\begin{figure}[htpb]
	\centering
	\resizebox{0.48\textwidth}{!}{
	\begin{tikzpicture}[ ->,>= stealth,minimum size=\nodeSize, shorten >=1pt,node distance=1.2cm,on grid,auto]
	\node[state] (A) [] {A};
	\node[node distance=2] (EPS1) [left=of A] { Episode 1:};
	\node[position=\angL:{\nodeDist} from A,circle,fill=black,inner sep=0pt,minimum size=7pt,label=left:{L}] (action_node_A_L) [] {};
	\node [draw,position=\angR:{\nodeDist} from A, circle,fill=green,inner sep=0pt,minimum size=7pt,label=left:{R}] (action_node_A_R)  {};
	\node at ($(action_node_A_L)  - (\distxL,\distyL)$)  (QL)  {$Q_n$ :};
	\node [node distance=\distSep][below=of QL] (L)  {$L_n$ :};
	\node [node distance=\distSep][below=of L] (Q)  {$Q'_n$ :};
	\node [node distance=\distSep][below=of Q] (U)  {$U_n$ :};
	\node [node distance=\distyL,below=of action_node_A_L] (action_node_A_L_QL) {$0$};
	\node [node distance=\distyL][below=of action_node_A_L_QL.east, anchor=east] (action_node_A_L_L) {$-20$};
	\node [node distance=\distSep][below=of action_node_A_L_L.east,anchor=east] (action_node_A_L_Q) {$0$};
	\node [node distance=\distSep][below=of action_node_A_L_Q.east, anchor=east] (action_node_A_L_U) {$20$};
	\node [node distance=\distyL][below=of action_node_A_R] (action_node_A_R_QL) {$-0.1$};
	\node [node distance=\distyL][below=of action_node_A_R_QL.east, anchor=east] (action_node_A_R_L) {$-18.9$};
	\node [node distance=\distSep][below=of action_node_A_R_L.east, anchor=east] (action_node_A_R_Q) {$-0.1$};
	\node [node distance=\distSep][below=of action_node_A_R_Q.east, anchor=east] (action_node_A_R_U) {$19.3$};
	\node[state,node distance=\distBetwnStates] (B) [right=of A] {B};
	\node[position=\angL:{\nodeDist} from B,circle,fill=black,inner sep=0pt,minimum size=7pt,label=left:{L}] (action_node_B_L) [] {};
	\node [draw,position=\angR:{\nodeDist} from B, circle,fill=green,inner sep=0pt,minimum size=7pt,label=left:{R}] (action_node_B_R)  {};
	\node [node distance=\distyL][below=of action_node_B_L] (action_node_B_L_QL) {$0$};
	\node [node distance=\distyL][below=of action_node_B_L_QL.east, anchor=east] (action_node_B_L_L) {$-20$};
	\node [node distance=\distSep][below=of action_node_B_L_L.east, anchor=east] (action_node_B_L_Q) {$0$};
	\node [node distance=\distSep][below=of action_node_B_L_Q.east, anchor=east] (action_node_B_L_U) {$20$};
	\node [node distance=\distyL][below=of action_node_B_R] (action_node_B_R_QL) {$-0.1$};
	\node [node distance=\distyL][below=of action_node_B_R_QL.east, anchor=east] (action_node_B_R_L) {$-18.98$};
	\node [node distance=\distSep][below=of action_node_B_R_L.east, anchor=east] (action_node_B_R_Q) {$-0.1$};
	\node [node distance=\distSep][below=of action_node_B_R_Q.east, anchor=east] (action_node_B_R_U) {$-19.02$};
	\node[state,node distance=1.5*\distBetwnStates] (nA) [right=of B] {A};
	\node[node distance=2] (EPSn) [left=of nA] { Episode $58:$};
	\node[position=\angL:{\nodeDist} from nA,circle,fill=black,inner sep=0pt,minimum size=7pt,label=left:{L}] (action_node_nA_L) [] {};
	\node [draw,position=\angR:{\nodeDist} from nA, circle,fill=green,inner sep=0pt,minimum size=7pt,label=left:{R}] (action_node_nA_R)  {};
	\node [node distance=\distyL][below=of action_node_nA_L] (action_node_nA_L_QL) {$-1.05$};
	\node [node distance=\distyL][below=of action_node_nA_L_QL.east, anchor=east] (action_node_nA_L_L) {$-5.56$};
	\node [node distance=\distSep][below=of action_node_nA_L_L.east, anchor=east] (action_node_nA_L_Q) {$-1.05$};
	\node [node distance=\distSep][below=of action_node_nA_L_Q.east, anchor=east] (action_node_nA_L_U) {$4.08$};
	\node [node distance=\distyL][below=of action_node_nA_R] (action_node_nA_R_QL) {$0.15$};
	\node [node distance=\distyL][below=of action_node_nA_R_QL.east, anchor=east] (action_node_nA_R_L) {$-0.12$};
	\node [node distance=\distSep][below=of action_node_nA_R_L.east, anchor=east] (action_node_nA_R_Q) {$0.07$};
	\node [node distance=\distSep][below=of action_node_nA_R_Q.east, anchor=east] (action_node_nA_R_U) {$0.07$};
	\node[state,node distance=\distBetwnStates] (nB) [right=of nA] {B};
	\node[position=\angL:{\nodeDist} from nB,circle,fill=black,inner sep=0pt,minimum size=7pt,label=left:{L}] (action_node_nB_L) [] {};
	\node [draw,position=\angR:{\nodeDist} from nB, circle,fill=green,inner sep=0pt,minimum size=7pt,label=left:{R}] (action_node_nB_R)  {};
	\node [node distance=\distyL][below=of action_node_nB_L] (action_node_nB_L_QL) {$-0.87$};
	\node [node distance=\distyL][below=of action_node_nB_L_QL.east, anchor=east] (action_node_nB_L_L) {$-8.29$};
	\node [node distance=\distSep][below=of action_node_nB_L_L.east, anchor=east] (action_node_nB_L_Q) {$-0.87$};
	\node [node distance=\distSep][below=of action_node_nB_L_Q.east, anchor=east] (action_node_nB_L_U) {$4.98$};
	\node [node distance=\distyL][below=of action_node_nB_R] (action_node_nB_R_QL) {$-0.12$};
	\node [node distance=\distyL][below=of action_node_nB_R_QL.east, anchor=east] (action_node_nB_R_L) {$-0.08$};
	\node [node distance=\distSep][below=of action_node_nB_R_L.east, anchor=east] (action_node_nB_R_Q) {$-0.08$};
	\node [node distance=\distSep][below=of action_node_nB_R_Q.east, anchor=east] (action_node_nB_R_U) {$0.05$};
	\path[->]
	(A)
	edge [dashed] node {} (action_node_A_L)
	(A)
	edge [] node {} (action_node_A_R)
	(B)
	edge [dashed] node {} (action_node_B_L)
	(B)
	edge [] node {} (action_node_B_R)
	(nA)
	edge [dashed] node {} (action_node_nA_L)
	(nA)
	edge [] node {} (action_node_nA_R)
	(nB)
	edge [dashed] node {} (action_node_nB_L)
	(nB)
	edge [] node {} (action_node_nB_R)
	($(A) + (\distBetwnStates/3,0)$)
	edge [] node {\small next iter.} ($(B) - (\distBetwnStates/3,0)$)
	($(nA) + (\distBetwnStates/3,0)$)
	edge [] node {\small next iter.} ($(nB) - (\distBetwnStates/3,0)$)
	;
	\end{tikzpicture}
	}
	\caption{An illustration of LBQL iterates for Example 1.}
	\label{F:Example_alg}
\end{figure}
\textit{We let $\alpha_n=0.1$, $\beta_n=0.05$ for all $n$. 
Figure \ref{F:Example_alg} illustrates two iterations from the first and the $58$th episodes.
Initially $Q_0(s,a)=0$ and $\rho=20$.
After one episode the bounds are still loose, so we have $Q_1(A,R)=Q'_1(A,R)={-0.1}$.
At episode 58 (281 iterations): learning has occurred for the lower and upper bounds values for the right action at \texttt{A} and \texttt{B}. 
We see that the bounds are effective already in 
keeping the $Q$-iterate close to $Q^*$. Interestingly, the upper bound is enforced at \texttt{A}, while the lower bound is enforced at \texttt{B}. Note that these are the results of a real simulation.}
\subsection{Analysis of Convergence}
In this section, we analyze the convergence of the idealized version of the LBQL algorithm to the optimal action-value function $Q^*$. We start by summarizing and developing some important technical results that will be used in our analysis. 
All proofs are presented in Appendix \ref{Appendix:Proofs}. 

The following proposition establishes the boundedness of the action-value iterates and  asymptotic bounds on the $L_n$ and $U_n$ iterates of Algorithm \ref{A:IV-LBQL}, which are needed in our proof of convergence. The proof of this proposition is presented in Section \ref{Appendix:LQU} in the Appendix.
\begin{prop} [Boundedness]
For all $(s,a) \in \mathcal{S} \times \mathcal{A}$, we have the following: 
\vspace{-1em}
\begin{enumerate}[label={(\roman*)},itemindent=0.em]
    \item The iterates $Q_n(s, a)$ and $Q'_n(s, a)$, remains bounded for all  $(s,a) \in \mathcal{S} \times \mathcal{A}$ and for all $n$.
    \label{L:BQ'}
    \vspace{-0.5em}
    \item For every $\eta > 0$, and with probability one, there exists some finite iteration index $n_0$ such that 
    \vspace{-0.7em}
    \begin{equation*}
        L_n(s,a) \le Q^*(s,a) + \eta \ \ \text{and} \ \ Q^*(s,a) - \eta \le U_n(s,a),
        \vspace{-0.7em}
    \end{equation*}
    for all iterations $n \ge n_0$. \label{L:LQU2}
\end{enumerate}
\label{L:LQU}
\end{prop}
\vspace{-1.0em}
Proposition \ref{L:LQU}\ref{L:BQ'} ensures that at each iteration $n$ the action-value iterates $Q_n$ and $Q'_n$ 
are bounded. 
This allows us to set $\varphi =  Q_{n+1}$ at each iteration of Algorithm \ref{A:IV-LBQL} and is required to establish convergence in general.
The proof is based on showing an inductive relationship that connects $Q_{n}$ and $Q'_{n}$ to the previous lower and upper bound iterates.
Specifically, we show that both action-value iterates are bounded below by the preceding upper bound iterates and above by the preceding lower bound iterates.
Proposition \ref{L:LQU}\ref{L:LQU2} ensures that 
there exists a finite iteration after which the lower and upper bound iterates $L_{n}$ and $U_{n}$ are lower and upper bounds on the optimal action-value function $Q^*$ with an error margin of at most an arbitrary amount $\eta > 0$. In the proof of Proposition \ref{L:LQU}\ref{L:LQU2}, we bound the lower and upper bound iterates by a noise process and another sequence that converges to $Q^*$. We show that the noise process possesses some properties that help to eliminate the effect of the noise asymptotically. With the effects of the noise terms vanishing, the boundedness of the lower and upper bound iterates by $Q^*$ is achieved.
Examining the update equations (\ref{Uupdate}) and (\ref{Lupdate}) for $U_{n+1}$ and $L_{n+1}$ in Algorithm \ref{A:IV-LBQL}, we remark that they are not ``standard'' stochastic approximation or stochastic gradient updates because $\hat{Q}_0^U$ and $\hat{Q}_0^L$ are computed with iteration-dependent penalty functions generated by $\varphi = Q_{n+1}$. In other words, the noiseless function itself is changing over time. The proof of Proposition \ref{L:LQU}\ref{L:LQU2} essentially uses the fact that even though these updates are being performed with respect to different underlying functions, as long as we can apply Proposition \ref{P:EmpDR} in every case, then after the noise is accounted for, the averaged values $U_{n+1}$ and $L_{n+1}$ are eventually bounded below and above by $Q^*$, respectively.
The following lemma derives some guarantees on the lower and upper bound iterates of Algorithm \ref{A:IV-LBQL}, whose proof appears in Section \ref{Appendix:LU} of the Appendix.

\begin{lem}[Consistency of Bounds]  
If $L_0(s,a) \le U_0(s,a)$, then $L_n(s,a) \le U_n(s,a)$ for all iterations $n$ and for all  $(s,a) \in \mathcal{S} \times \mathcal{A}$.
\label{L:LU} 
\end{lem}
\vspace{-0.7em}
In particular, Lemma \ref{L:LU} shows that the upper and lower bound iterates do not interchange roles and become inconsistent. This is an important property; otherwise, the projection step of Algorithm \ref{A:IV-LBQL} loses its meaning and would require additional logic to handle inconsistent bounds.
The results of Lemma \ref{L:LU} follows mainly by the fact that we are using the same sample path to solve the upper and lower bound inner problems, \eqref{E:EmpQG} and \eqref{E:EmpQL}, respectively. 
Before stating our convergence results, we first state a typical assumption on the stepsizes and the state visits.
\begin{assumption}
We assume that: 
\vspace{-0.7em}
\begin{enumerate}[label={(\roman*)},itemindent=0em,leftmargin=*]
\item $\sum_{n=0}^{\infty} \alpha_n(s,a) = \infty, \,
\sum_{n=0}^{\infty} \alpha^2_n(s,a) < \infty, \\[2pt]
\sum_{n=0}^{\infty} \beta_n(s,a) = \infty, \, 
\sum_{n=0}^{\infty} \beta^2_n(s,a) < \infty,$
\vspace{-0.2em}
\item Each state $s \in \mathcal{S}$ is visited infinitely often with probability one.
\vspace{-0.5em}
\end{enumerate}
\label{Assumption}
\end{assumption}
\vspace{-0.7em}
We now state one of our main theoretical results. 
\begin{thm}[Convergence of LBQL]
Under Assumption \ref{Assumption}, the following hold with probability 1:
\vspace{-0.7em}
\begin{enumerate}[label={(\roman*)},itemindent=0em,leftmargin=*]
    \item $ Q'_n(s,a)$ in Algorithm \ref{A:IV-LBQL} converges to the optimal action-value function $Q^*(s,a)$ for all state-action pairs $(s,a)$. 
    \label{T:1}
    \vspace{-0.5em}
    \item If the penalty terms are computed exactly, i.e. as per \eqref{E:za}, then  the iterates $L_n(s,a), Q'_n(s,a), U_n(s,a)$ in Algorithm \ref{A:IV-LBQL} converge to the optimal action-value function $Q^*(s,a)$ for all state-action pairs $(s,a)$. \label{T:2}
\end{enumerate}
\label{T:Conv}
\end{thm}
\vspace{-0.7em}
Due to the interdependent feedback between $Q$, $U$, and $L$, it is not immediately obvious that the proposed scheme does not diverge. The primary challenge in the analysis for this theorem is to handle this unique aspect of the algorithm.


\subsection{LBQL with Experience Replay}
\label{S:LBQL_ER}
We now introduce a more practical version of LBQL that uses experience replay in lieu of a black-box simulator.
Here, we use a noise buffer $\mathcal B$ to record the unique noise values $w$ that are observed at every iteration. 
We further assume that the noise space $\mathcal W$ is finite, a reasonable assumption for a finite MDP. The buffer $\mathcal B$ is used in two ways: (1) to generate the sample path $\mathbf w$ and (2) to estimate the expectation in the penalty function. 
Here, we track and update the \emph{distribution} of the noise $w$ after every iteration and directly compute the expectation under this distribution instead of sampling a batch of size $K$, as we did previously. 
To illustrate how this can be done, suppose $\mathcal W=\{w_a, w_b, w_c, w_d\}$ and that at iteration $n$ we observe $w_{n+1}=w_a$. 
Let $p_a$ denote the probability of observing $w_a$, and $N_n(w_a)$ the number of times $w_a$ is observed in the first $n$ iterations, then the empirical estimate of $p_a$ is given by $\hat p_n(w_a) = N_n(w_a)/n$.\footnote{Note that LBQL could, in principle, be adapted to the case of of continuous noise (i.e., where $w$ is continuous random variable) using methods like kernel density estimation (KDE).}
We denote by $\hat \E_n[\,.\,]$ the expectation computed using the empirical distribution $\hat p_n$.
To differentiate the penalty and the action-values (solutions to the inner problems) that are computed from the buffer from those defined in the idealized version of the algorithm, we define:
\begin{equation}
\begin{split}
\textstyle \tilde \zeta_t^\pi (s_t, a_t, &w \, | \, \varphi) :=\varphi(s_{t+1},\pi(s_{t+1}))  \\ &- \hat \E_n [ \varphi \bigl(h(s_t,a_t,w),\pi(h(s_t,a_t,w))\bigr)],
\end{split}
\label{E:Empza_biased}
\end{equation}
and given a sample path $\mathbf w=(w_1, w_2, \ldots,w_{\tau})$ the inner problems analogous to  \eqref{E:EmpQG} and \eqref{E:EmpQL} are given by 
\begin{align}
&\begin{split}
 \tilde Q_t^{U}(s_t,a_t)  & =  r(s_t,a_t)  -\tilde \zeta_t^{\pi_{\varphi}} (s_t, a_t, w_{t+1} \, | \, \varphi) \\ & + \max_{a}  \tilde Q_{t+1}^{U}(s_{t+1},a) \label{E:EmpQG_biased}
 \end{split}\\
 &\begin{split}
 \tilde Q^L_t(s_t,a_t) &=  r(s_t,a_t)  -\tilde \zeta_t^{\pi_{\varphi}} (s_t, a_t, w_{t+1} \, | \, \varphi)  \\ & + \tilde Q^L_{t+1}(s_{t+1},\pi_{\varphi}(s_{t+1}))
  \end{split}
\label{E:EmpQL_biased}
\end{align}
for $t=0,1,\ldots,\tau-1,$ where $s_{t+1}= h(s_t, a_t, w_{t+1})$ and $\tilde Q^{ U}_{\tau}\equiv \tilde Q^L_{\tau}\equiv0$. The pseudo-code of LBQL with experience replay is shown in Algorithm \ref{A:ER-LBQL} in Appendix \ref{Appendix:LBQLwithExpRep}.

\subsection{Convergence of LBQL with Experience Replay}

In this section, we prove that the version of LBQL with experience replay also converges to the optimal action-value function.
We start by stating a lemma that confirms Proposition \ref{L:LQU} and Lemma \ref{L:LU} still hold when the penalty terms are computed using \eqref{E:Empza_biased}.
\begin{lem}
If at any iteration $n$, the penalty terms are computed using the estimated distribution $\hat p_n$, i.e.,
as per \eqref{E:Empza_biased},
then Proposition \ref{L:LQU} and Lemma \ref{L:LU} still hold.
\label{L:pLQL2}
\end{lem}

\begin{thm}[Convergence of LBQL with experience replay]
Under Assumption \ref{Assumption}, the following hold with probability 1:
 \vspace{-1em}
\begin{enumerate}[label={(\roman*)},itemindent=0em,leftmargin=*]
    \item  $ Q'_n(s,a)$ in Algorithm \ref{A:ER-LBQL} converges to the optimal action-value function $Q^*(s,a)$ for all state-action pairs $(s,a)$. 
    \label{T:ER1}
    \item The iterates $L_n(s,a), Q'_n(s,a),$ $U_n(s,a)$ in Algorithm \ref{A:ER-LBQL} converge to the optimal action-value function $Q^*(s,a)$ for all state-action pairs $(s,a)$. \label{T:ER2}
\end{enumerate}
\label{T:ConvERLBQL}
\end{thm}
\vspace{-0.5em}
The proof is similar to that of Theorem \ref{T:Conv}, but using the observations collected in the buffer naturally results in an additional bias term in our analysis. The proof of Lemma \ref{L:pLQL2} shows that as we learn the distribution of the noise, this bias term goes to zero and our original analysis in the unbiased case continues to hold.

Notice that the results in part \ref{T:ER2} of the theorem are, in a sense, stronger than that of Theorem \ref{T:Conv}\ref{T:2}. While both achieve asymptotic convergence of the lower and upper bounds to the optimal action-value function,  Theorem \ref{T:ConvERLBQL}\ref{T:ER2} does not require computing the penalty with the true distribution, i.e., using \eqref{E:za}. This is because in the experience replay version, the distribution of the noise random variables is also learned.

\section{Numerical Experiments}
\label{S:NE}
In our numerical experiments we make slight modifications to Algorithm \ref{A:ER-LBQL}, which help to reduce its computational requirements. 
A detailed description of all changes is included in Appendix \ref{Appendix:LBQLwERSV}.
We also open-source a Python package\footnote{\href{https://github.com/ibrahim-elshar/LBQL_ICML2020}{https://github.com/ibrahim-elshar/LBQL\_ICML2020}.} for LBQL that reproduces all experiments and figures presented in this paper. We compare LBQL with experience replay with several algorithms: Q-learning (QL), double Q-learning (Double-QL), speedy Q-learning (SQL), and bias-corrected Q-learning (BCQL) \citep{hasselt2010double,azar2011speedy,lee2019bias}. 
The environments that we consider are summarized below. 
Detailed description of the environments, the parameters used for the five algorithms, and sensitivity analysis are deferred to Appendix \ref{Appendix:Experiments}.

    \textbf{Windy Gridworld (WG).} This is a well-known variant of the standard gridworld problem discussed in \citet{sutton2018reinforcement}. There is an \emph{upward wind} with a random intensity. The agent moves extra steps in the wind direction whenever it reaches an affected square. The reward is $-1$ until the goal state is reached, and the reward is 0 thereafter.
    \\[0.5em]
    \textbf{Stormy Gridworld (SG).} We then consider a new domain that adds the additional complexity of rain and \emph{multi-directional} wind to windy gridworld. The location of the rain is random and when it occurs, puddles that provide negative rewards are created. The reward is similar to that of WG, except that puddle states provide a reward of $-10$.
\begin{figure*}[htpb] 
\centering
	\begin{subfigure}{0.33\textwidth}
		\includegraphics[width=\linewidth]{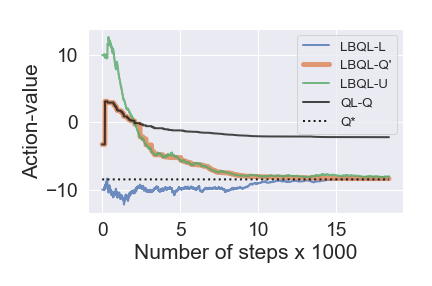}
		\caption{WG} \label{F:WGBounds}
	\end{subfigure}
	\hspace{-0.2cm}
	\begin{subfigure}{0.33\textwidth}
		\includegraphics[width=\linewidth]{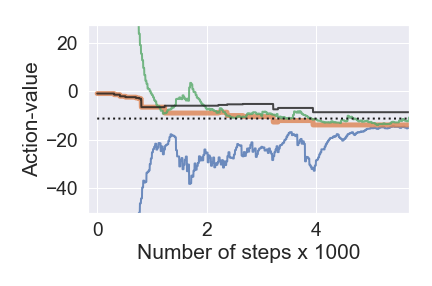}
		\caption{SG} \label{F:RBounds}
	\end{subfigure}
	\begin{subfigure}{0.33\textwidth}
		\includegraphics[width=\linewidth]{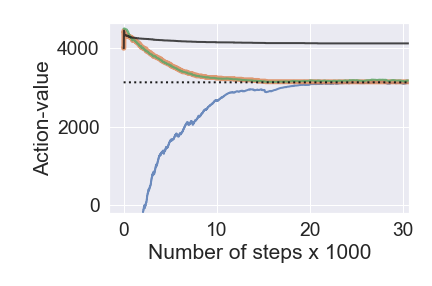}
		\caption{2-CS-R} \label{F:CSRbounds}
	\end{subfigure}
	\begin{subfigure}{0.33\textwidth}
		\includegraphics[width=\linewidth]{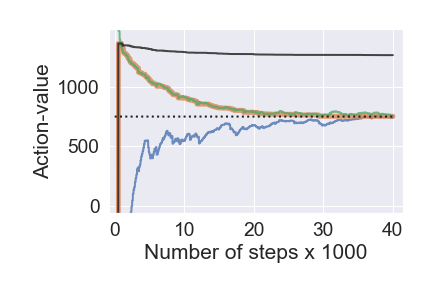}
		\caption{2-CS} \label{F:CSbounds}
	\end{subfigure}
	\hspace{-0.2cm}
		\begin{subfigure}{0.33\textwidth}
		\includegraphics[width=\linewidth]{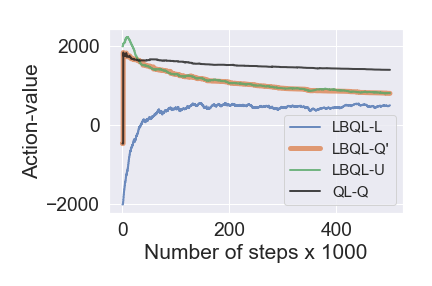}
		\caption{4-CS} \label{F:CSbounds4s}
	\end{subfigure}
\caption{Illustration of LBQL Upper and Lower Bounds.} \label{F:Bounds}
\end{figure*}
\\[0.5em]
\textbf{Repositioning in Two-Location Car-sharing (2-CS-R).} Our next benchmark is a synthetic problem of balancing an inventory of cars by repositioning them in a car-sharing platform with two stations \citep{he2019robust}. The actions are to decide on the number of cars to be repositioned from one station to the other before random demand is realized. All rentals are one-way (i.e., rentals from station A end up at B, and vice-versa). The goal is to maximize revenue for a fixed rental price subject to lost sales and repositioning costs.
\\[0.5em]
\textbf{Pricing in Two-Location Car-sharing (2-CS).} Here, we consider the benchmark problem of spatial dynamic pricing on a car-sharing platform with two stations, motivated partially by \cite{bimpikis2019spatial}. The actions are to set a price at each station, which influence the station's (stochastic) demand for rentals. Rentals are one-way and the goal is to maximize revenue under lost sales cost.
 \\[0.5em]
\textbf{Pricing in Four-Location Car-sharing (4-CS).} The final benchmark that we consider is a variant of the above pricing problem with four stations. Now, however, we consider both one way and return trips at each station.
In this case, we have two sources of randomness: the noise due to stochastic demand and the noise due to the random distribution of fulfilled rentals between the stations.
\\[0.5em]
First, we illustrate conceptually in Figure \ref{F:Bounds} how the upper and lower bounds of LBQL can ``squeeze'' the Q-learning results toward the optimal value (the plots show a particular state-action pair $(s,a)$ for illustrative reasons). For example, in Figure \ref{F:WGBounds}, we observe that the LBQL iterates (orange) match the Q-learning iterates (solid black) initially, but as the upper bound (green) becomes better estimated, the LBQL iterates are pushed toward the optimal value (dotted black).
We see that even though the same hyperparameters are used between LBQL and QL, the new approach is able to quickly converge. 
In the 4-CS example, Figure \ref{F:CSbounds4s}, $Q^*$ is not shown since it is computationally difficult to obtain, but the gap between the upper and lower bounds, along with Theorem \ref{T:ConvERLBQL}\ref{T:ER2}, suggest that LBQL is converging faster than standard Q-learning.


\begin{figure*}[htb] 
\centering
	\begin{subfigure}{0.33\textwidth}
		\includegraphics[width=\linewidth]{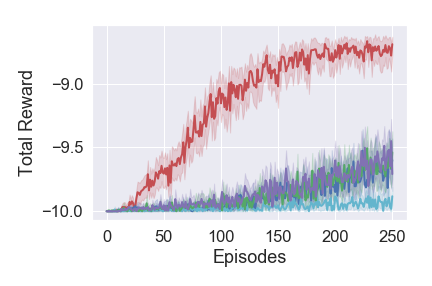}
		\vspace{-10pt}
		\caption{Performance (WG)} \label{F:WGPerformance}
	\end{subfigure}	
	\begin{subfigure}{0.33\textwidth}
		\includegraphics[width=\linewidth]{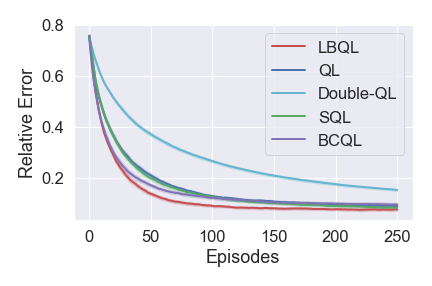}
		\vspace{-10pt}
		\caption{Relative Error (WG)} \label{F:WGRE}
	\end{subfigure}
	\begin{subfigure}{0.33\textwidth}
		\includegraphics[width=\linewidth]{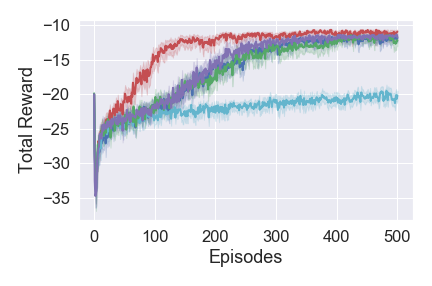}
				\vspace{-10pt}
		\caption{Performance (SG)} \label{F:RGWPerformance}
	\end{subfigure}
		\begin{subfigure}{0.33\textwidth}
		\includegraphics[width=\linewidth]{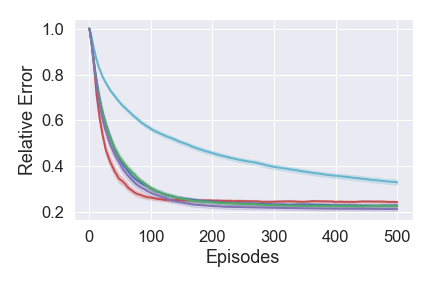}
				\vspace{-10pt}
		\caption{Relative Error (SG)} \label{F:RGWRE}
	\end{subfigure}
		\begin{subfigure}{0.33\textwidth}
		\includegraphics[width=\linewidth]{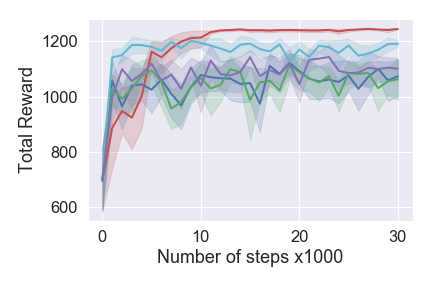}
		\vspace{-10pt}
		\caption{Performance (2-CS-R)} \label{F:CSRP}
	\end{subfigure}
	\begin{subfigure}{0.33\textwidth}
		\includegraphics[width=\linewidth]{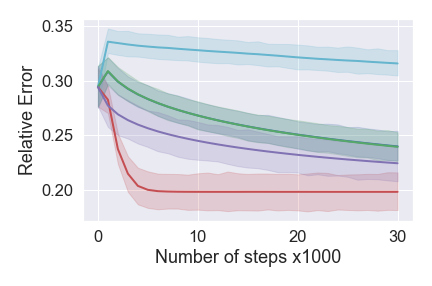}
		\vspace{-10pt}
		\caption{Relative Error (2-CS-R)} \label{F:CSRRE}
	\end{subfigure}
	\begin{subfigure}{0.33\textwidth}
		\includegraphics[width=\linewidth]{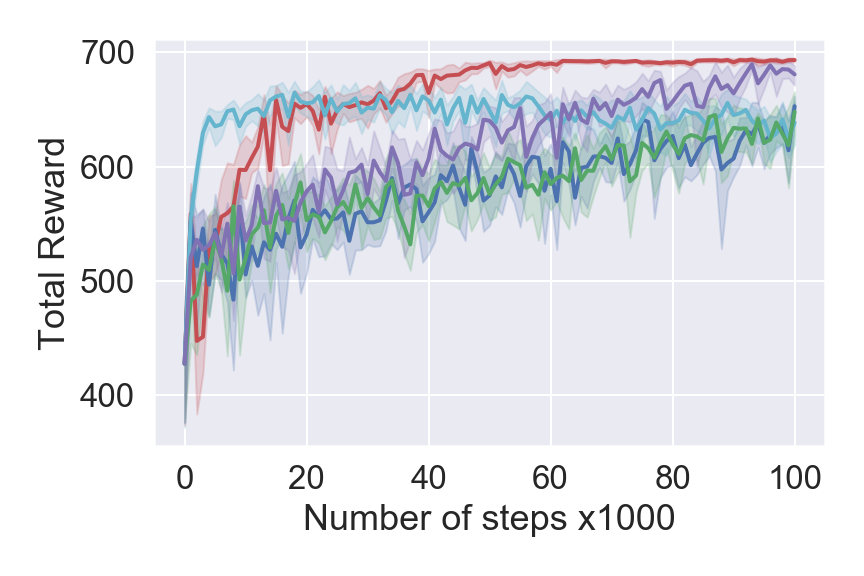}
		\vspace{-10pt}
		\caption{Performance (2-CS)} \label{F:CSP}
	\end{subfigure}
	\begin{subfigure}{0.33\textwidth}
		\includegraphics[width=\linewidth]{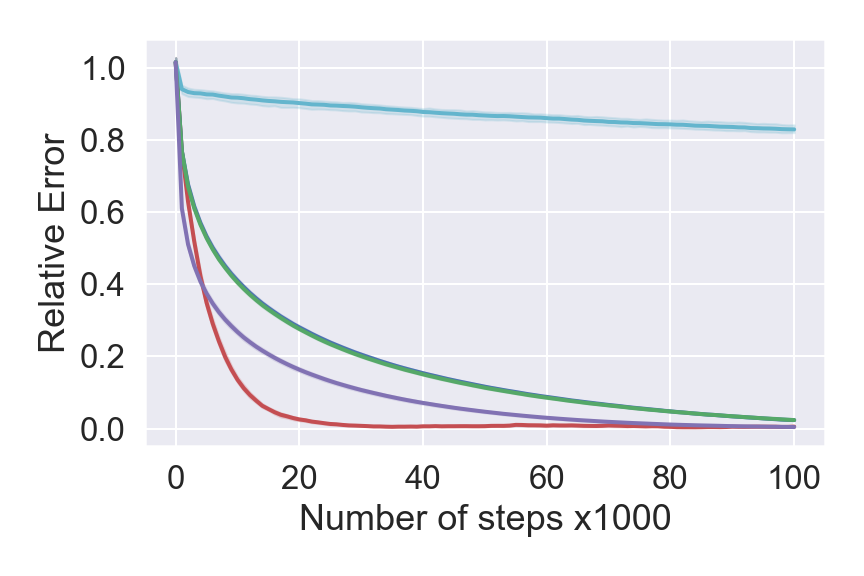}
		\vspace{-10pt}
		\caption{Relative Error (2-CS)} \label{F:CSRE}
	\end{subfigure}
	\begin{subfigure}{0.33\textwidth}
		\includegraphics[width=\linewidth]{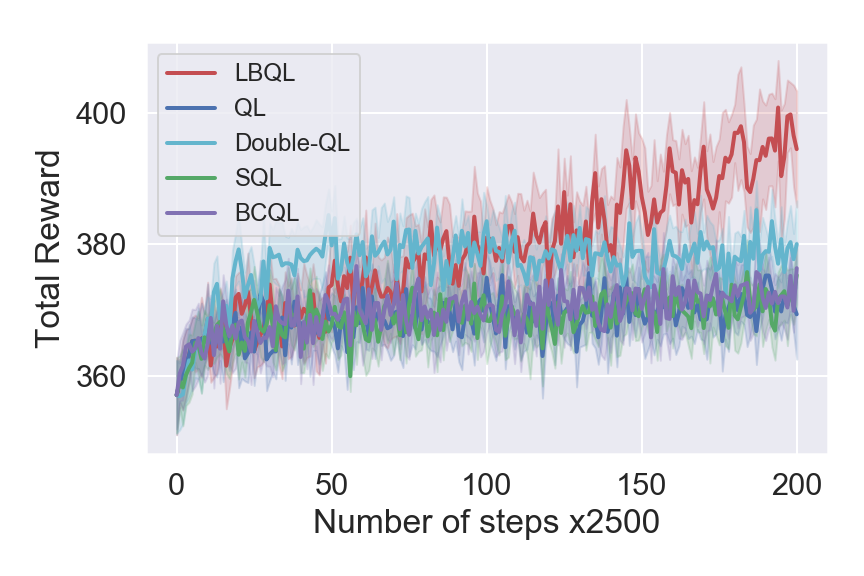}
		\vspace{-10pt}
		\caption{Performance (4-CS)} \label{F:CS4}
	\end{subfigure}
	\caption{Results from Numerical Experiments.}
		\label{F:all}
\end{figure*}	
The full results (with $95\%$ confidence intervals) of the numerical experiments are shown in Figure \ref{F:all}.
LBQL drastically outperforms the other algorithms in terms of the performance curve on the gridworld domains, but for the car-sharing problems, double Q-learning is superior in the first 20,000 steps. Afterwards, LBQL catches up and remains the best performing algorithm. From the relative error plots (which measure the percent error, in $l_2$-norm, of the approximate value function with the true optimal value function, i.e., $\|V_n - V^* \|_2/\|V^*\|_2$), we see that LBQL has the steepest initial decline. In windy gridworld and car-sharing, LBQL outperforms other algorithms in terms of relative error, but BCQL and SQL achieve slightly lower relative error than LBQL for stormy gridworld.

We also conducted a set of sensitivity analysis experiments, where we varied the learning rate and exploration hyperparameters across all algorithms (results are given in Appendix \ref{Appendix:SA}). We examine the number of iterations and CPU time needed to reach 50\%, 20\%, 5\%, and 1\% relative error. The results show that LBQL outperforms BCQL, SQL, and QL in terms of both iterations and CPU time in reaching 20\%, 5\%, and 1\% relative error across all 15 hyperparameter settings that we tested. For the case of 50\% relative error, BCQL outperforms LBQL in five out of 15 cases. This indicates that LBQL is significantly more robust to hyperparameter settings than the other algorithms.
Roughly speaking, this robustness might be attributed to the ``approximate planning'' aspect of the algorithm, where lower and upper bounds are computed.
\section{Conclusion}
\label{S:Conclusion}
In this paper, we present the LBQL algorithm and prove its almost sure convergence to the optimal action-value function. We also propose a practical extension of the algorithm that uses an experience replay buffer. Numerical results illustrate the rapid convergence of our algorithm empirically when compared to a number of well-known variants of Q-learning on five test problems. LBQL is shown to have superior performance, robustness against learning rate and exploration strategies, and an ability to mitigate maximization bias.

Interesting future work is the extension of our new framework to the model-based RL setting, where the transition function $f$ is learned while the policy is optimized.
Other interesting future work includes looking beyond the tabular case and adapting our algorithm to the setting of value function approximations, such as DQN \citep{mnih2013playing}.

\bibliography{references}
\bibliographystyle{icml2020}
\clearpage

\appendix
 \label{A:A}
 \counterwithin{lem}{section}
\input{appendix.tex}

\end{document}

%% file: appendix.tex
\onecolumn
\begin{center}
\Large Appendix to \emph{Lookahead-Bounded Q-Learning}
\end{center}
\section{Proofs}
\label{Appendix:Proofs}
\allowdisplaybreaks

\subsection{Proof of Proposition \ref{P:EmpDR}}
\label{Appendix:ProofPropEmpDR}
\begin{proof}
We provide a proof that is similar to that of Proposition 2.3 (iv) of \cite{brown2010information} but for the case of the absorption time formulation of an infinite horizon problem.
Here, we define a policy $\pi:=\{\pi_t \}_{t\ge0}$ as a sequence of functions, that maps from $\{w_t\}_{t\ge1}$ to feasible actions. We may also use stationary policies where $\pi_t$ is the same for all $t$ and only depends on the current state $s_t$. 
Let $\mathbb G=\{\mathcal G_t \}_{t\ge 0}$ be the perfect information relaxation of the natural filtration $\mathbb F=\{ \mathcal F_t\}_{t\ge 0}$. Under $\mathbb G$, we have $\mathcal G_t =\mathcal F$, i.e.,  we have access to the entire future uncertainties at each $t$. 
Define by $\Pi_{\mathbb G}$ the set of policies that includes the policies that have access to future uncertainties in addition to nonanticipative policies. 
Let $\mathbb{\hat G}$ be a relaxation of $\mathbb G$ such that in addition to what is known under $\mathbb G$ the estimate penalty terms $\hat \zeta_t^{\pi_{\varphi}} (s_t, a_t, w_{t+1} \, | \, \varphi)$ are revealed at time $t$. 

We first prove $\E[\hat Q^L_0(s,a)] \le Q^*(s,a) $. For an admissible policy $\pi$, we have
\begin{equation*}
    \begin{split}
        \E[\hat Q^L_0(s,a)] & \myeq{(a)} \E \left[ \sum_{t=0}^{\tau -1}   r(s_t, \pi(s_t)) - \hat \zeta_t^{\pi}(s_t,a_t,w_{t+1} \, | \, \varphi)   \, | \, s_0=s, \, a_0=a \right] \\ & 
        \myeq{(b)}\E \left[  \sum_{t=0}^{\tau-1} (r(s_t, \pi(s_t)) -   \hat \zeta_t^{\pi} (s_t, a_t, w_{t+1} \, | \, \varphi)) \, \mathbbm{1}_{\{\tau< \infty\}} \, | \, s_0=s, \, a_0=a  \right] \\ &
        \myeq{(c)} \sum_{\tau'=1}^{\infty }\E \left[ \sum_{t=0}^{\tau'-1} (r(s_t, \pi(s_t)) -   \hat \zeta_t^{\pi} (s_t, a_t, w_{t+1} \, | \, \varphi)) \, \mathbbm{1}_{\{\tau=\tau'\}} \, | \, s_0=s, \, a_0=a  \right] \\ & 
        \myeq{(d)} \sum_{\tau'=1}^{\infty }\E \left[ \sum_{t=0}^{\tau'-1} (r(s_t, \pi(s_t)) -  \E[ \hat \zeta_t^{\pi} (s_t, a_t, w_{t+1} \, | \, \varphi) \, | \, \mathcal G_t]) \, \mathbbm{1}_{\{\tau=\tau'\}} \, | \, s_0=s, \, a_0=a  \right] \\ & 
        \myeq{(e)} \sum_{\tau'=1}^{\infty }\E \left[ \sum_{t=0}^{\tau'-1} (r(s_t, \pi(s_t)) -   \zeta_t^{\pi} (s_t, a_t, w_{t+1} \, | \, \varphi)) \, \mathbbm{1}_{\{\tau=\tau'\}} \, | \, s_0=s, \, a_0=a  \right] \\ & 
       \myeq{(f)} \E \left[  \sum_{t=0}^{\tau-1} (r(s_t, \pi(s_t)) -   \zeta_t^{\pi} (s_t, a_t, w_{t+1} \, | \, \varphi)) \, \mathbbm{1}_{\{\tau< \infty\}} \, | \, s_0=s, \, a_0=a  \right] \\ &
       \myeq{(g)} \E \left[  \sum_{t=0}^{\tau-1} (r(s_t, \pi(s_t)) -   \zeta_t^{\pi} (s_t, a_t, w_{t+1} \, | \, \varphi)) \, | \, s_0=s, \, a_0=a  \right] \\ &
       \le Q^*(s,a)
    \end{split}
\end{equation*}    
Equality (a) follows from the definition of $\hat Q_0(s,a)$ and equality (b) follows since $\tau$ has finite mean and $r$ and $\varphi$ are uniformly bounded. Equalities (c) and (d) follow from the law of total expectations. Equality (e) follows from Lemma A.1 in \cite{brown2010information} and from the estimated penalty terms being unbiased, i.e.,  $\E[\hat \zeta_t^{\pi_G^*} (s_t, a_t, w_{t+1} \, | \, \varphi) \, | \, \mathcal G_t]=  \zeta_t^{\pi_G^*} (s_t, a_t, w_{t+1} \, | \, \varphi)$. Equalities (f) and (g) follow by the law of total expectation and $\tau$ being almost surely finite stopping time, respectively. The inequality follows since the expected value of the penalty terms for a feasible policy is zero and the action-value function of a feasible policy, $Q^\pi(s,a)$, is less than $Q^*(s,a)$.

Now, we prove $Q^*(s,a) \le \E[\hat Q^U_0(s,a)]$. 
Let $\pi_G^*$ be the optimal solution for the dual problem,
\begin{equation}
\max_{\pi_G \in \Pi_\mathbb{ G}} \E \left [ \sum_{t=0}^{\tau -1} (r(s_t, \pi_G) -  \zeta_t^{\pi_G} (s_t, a_t, w_{t+1} \, | \, \varphi)) \, | \, s_0=s, \, a_0=a  \right]. 
\label{E:DualEP}
\end{equation}
We have,
\begin{equation*}
    \begin{split}
        \E[\hat Q^U_0(s,a)] & \myeq{(a)} \E \left[ \max_{\mathbf a} \left\{\sum_{t=0}^{\tau -1}     r(s_t, a_t) - \hat \zeta_t^{\pi_\varphi}(s_t,a_t,w_{t+1} \, | \, \varphi)  \right \} \, | \, s_0=s, \, a_0=a \right] \\ & 
        \myeq{(b)}\max_{\pi_G \in \Pi_\mathbb{\hat G}} \E \left [ \sum_{t=0}^{\tau -1} (r(s_t, \pi_G) -  \hat \zeta_t^{\pi_G} (s_t, a_t, w_{t+1} \, | \, \varphi)) \, | \, s_0=s, \, a_0=a  \right] \\ &
        \ge \E \left[  \sum_{t=0}^{\tau-1} (r(s_t, \pi_G^*) -   \hat \zeta_t^{\pi_G^*} (s_t, a_t, w_{t+1} \, | \, \varphi)) \, | \, s_0=s, \, a_0=a  \right] \\ &
       \myeq{(c)}\E \left[  \sum_{t=0}^{\tau-1} (r(s_t, \pi_G^*) -   \hat \zeta_t^{\pi_G^*} (s_t, a_t, w_{t+1} \, | \, \varphi)) \, \mathbbm{1}_{\{\tau< \infty\}} \, | \, s_0=s, \, a_0=a  \right] \\ &
       \myeq{(d)} \sum_{\tau'=1}^{\infty }\E \left[ \sum_{t=0}^{\tau'-1} (r(s_t, \pi_G^*) -   \hat \zeta_t^{\pi_G^*} (s_t, a_t, w_{t+1} \, | \, \varphi)) \, \mathbbm{1}_{\{\tau=\tau'\}} \, | \, s_0=s, \, a_0=a  \right] \\ & 
       \myeq{(e)} \sum_{\tau'=1}^{\infty }\E \left[ \sum_{t=0}^{\tau'-1} (r(s_t, \pi_G^*) -  \E[ \hat \zeta_t^{\pi_G^*} (s_t, a_t, w_{t+1} \, | \, \varphi) \, | \, \mathcal G_t]) \, \mathbbm{1}_{\{\tau=\tau'\}} \, | \, s_0=s, \, a_0=a  \right] \\ & 
        \myeq{(f)} \sum_{\tau'=1}^{\infty }\E \left[ \sum_{t=0}^{\tau'-1} (r(s_t, \pi_G^*) -   \zeta_t^{\pi_G^*} (s_t, a_t, w_{t+1} \, | \, \varphi)) \, \mathbbm{1}_{\{\tau=\tau'\}} \, | \, s_0=s, \, a_0=a  \right] \\ & 
       \myeq{(g)} \E \left[  \sum_{t=0}^{\tau-1} (r(s_t, \pi_G^*) -   \zeta_t^{\pi_G^*} (s_t, a_t, w_{t+1} \, | \, \varphi)) \, \mathbbm{1}_{\{\tau< \infty\}} \, | \, s_0=s, \, a_0=a  \right] \\ &
        \myeq{(h)} \E \left[  \sum_{t=0}^{\tau-1} (r(s_t, \pi_G^*) -   \zeta_t^{\pi_G^*} (s_t, a_t, w_{t+1} \, | \, \varphi)) \, | \, s_0=s, \, a_0=a  \right] \\ &
        \myeq{(i)} \max_{\pi_G \in \Pi_\mathbb{ G}} \E \left [ \sum_{t=0}^{\tau -1} (r(s_t, \pi_G) -  \zeta_t^{\pi_G} (s_t, a_t, w_{t+1} \, | \, \varphi)) \, | \, s_0=s, \, a_0=a  \right] \\ & \ge Q^*(s,a).
    \end{split}
\end{equation*}
Equality (a) and (b) follow from the definition of $\hat Q^U_0(s,a)$ and since $ \mathbb {\hat G}$ is a relaxation of the perfect information relaxation $\mathbb G$, which allows us to interchange the maximum and the expectation. The first inequality follows because $\pi_G^* \in \Pi_{\mathbb{\hat G}}$ since $\Pi_{\mathbb G} \subseteq \Pi_{\mathbb {\hat G}}$. 
Equality (c) follows since $r$ and $\varphi$ are uniformly bounded and $\tau$ has finite mean.
Equalities (d) and (e) follow from the law of total expectations. Equality (f) follows from Lemma A.1 in \cite{brown2010information} and from the estimated penalty terms being unbiased, i.e.,  $\E[\hat \zeta_t^{\pi_G^*} (s_t, a_t, w_{t+1} \, | \, \varphi) \, | \, \mathcal G_t]=  \zeta_t^{\pi_G^*} (s_t, a_t, w_{t+1} \, | \, \varphi)$. Equalities (g) and (h) follow by the law of total expectation and $\tau$ being almost surely finite stopping time, respectively. Equality (i) follows since by definition $\pi_G^*$ is the optimal solution of \eqref{E:DualEP}. The last inequality follows by weak duality (Proposition \ref{P:ATDR}\ref{P:WD}).
\end{proof}

First, we state a technical lemma that is used in the proof of Proposition \ref{L:LQU} and Lemma \ref{L:LU}.

\begin{lem}
For all $n=1,2,\ldots$, if $L_{n-1}(s,a) \le U_{n-1}(s,a)$ and $Q'_{n-1} \in \mathcal{Q}$ then $L_{n}(s,a) \le U_{n}(s,a)$ for all $(s,a) \in \mathcal{S} \times \mathcal{A}$.
\label{L:LQU1} 
\end{lem}
\begin{proof}
Fix an $(s,a) \in \mathcal S \times \mathcal A$.
Note that the optimal values of the inner problems in \eqref{E:EmpQG} and \eqref{E:EmpQL}, $\hat Q^U_{0}(s,a)$ and $\hat Q^L_{0}(s,a)$ respectively, are computed using the same sample path $\mathbf w$ and for each period within the inner DP, the same batch of samples is used for estimating the expectation in both the upper and lower bound problems. 
For clarity, let us denote the values of $\hat Q^L_{0}(s,a)$ and $\hat Q^U_{0}(s,a)$ at iteration $n=1,2,\ldots$ by $\hat Q^L_{n,0}(s,a)$ and $\hat Q^U_{n,0}(s,a)$, respectively.
Assume $\alpha_n(s,a) \le 1$ for all $n$.
We provide a proof by induction. For $n=1$, we have:
\begin{equation*}
Q_{1}(s_0,a_0) = Q'_0(s_0,a_0) 
+ \alpha_0(s_0, a_0) \Bigl[ r(s_0,a_0)+ \gamma \max_a Q'_0(s_{1},a) - Q'_0(s_{0},a_0) \Bigr ].
\end{equation*}
Since 
the rewards $r(s,a)$ are uniformly bounded, $0< \gamma < 1$  and $|Q'_0(s,a)|\le \rho$ then $Q_{1}$ is bounded. Set $\varphi=Q_1$, since the actions selected by the policy $\pi_{Q_{1}}$ are feasible in \eqref{E:EmpQG}, we have 
\[
\hat Q_{1,0}^U(s,a) - \hat Q_{1,0}^L(s,a) \ge 0,
\]
and with $L_{0}(s,a) \le U_{0}(s,a)$, it follows that $L_{1}(s,a) \le U_{1}(s,a)$.
A similar proof can be used to show the inductive case also holds at iteration $n$,
\begin{equation*}
\begin{split}
  Q_{n}(&s_{n-1},a_{n-1}) = Q'_{n-1}(s_{n-1},a_{n-1}) \\ &
  + \alpha_{n-1}(s_{n-1}, a_{n-1}) \Bigl[ r(s_{n-1},a_{n-1})+ \gamma \max_a Q'_{n-1}(s_{n},a) - Q'_{n-1}(s_{n-1},a_{n-1}) \Bigr ].
  \end{split}
\end{equation*}
By the inductive hypothesis, we have $Q'_{n-1} \in \mathcal Q$ and $L_{n-1}(s,a) \le U_{n-1}(s,a)$. Then, similar to the base case, we have $Q_{n} \in \mathcal Q$ and 
$
\hat Q_{n,0}^U(s,a) - \hat Q_{n,0}^L(s,a) \ge 0.
$
Therefore, $L_{n}(s,a) \le U_{n}(s,a)$. Since our choice of $(s,a)$ was arbitrary then the result follows for all $(s,a) \in \mathcal S \times \mathcal A$.
\end{proof}

\subsection{Proof of Proposition \ref{L:LQU}}
\label{Appendix:LQU}
\begin{proof}
Part \ref{L:BQ'}: First, note that by \eqref{Uupdate} and  \eqref{Lupdate} the upper and lower bound estimates $U_n(s,a)$ and $L_n(s,a)$ are bounded below and above by 
$\rho$ and $-\rho$
respectively for all $(s,a) \in \mathcal S \times \mathcal A$ and for all n, 
where $\rho = R_{\max}/(1-\gamma)$. 
We assume in this proof that $\alpha_n(s,a) \le 1$ for all $n$. 
Let $\tilde{L}_n= \max_{(s,a)} L_n(s,a)$ and $\tilde{U}_n= \max_{(s,a)} U_n(s,a)$.
We claim that for every iteration $n$, we have that for all $(s,a)$,
\begin{equation}
     \bar{L}_n \le Q_n(s,a) \le \bar{U}_n \quad \text{and} \quad \bar{L}'_n \le Q'_n(s,a) \le \bar{U}'_n
     \label{C:LQ'U}
\end{equation}
where
\begin{align}
    \bar{L}_n &=\min \textstyle \left\{\tilde{U}_{n-1}(1+\gamma),\ldots, \tilde{U}_1 \sum_{i=0}^{n-1} \gamma^i, -M \sum_{i=0}^{n} \gamma^i \right\}
    \label{E:Lbar} 
     \\
    \bar{U}_n &= \textstyle \max \left\{\tilde L_{n-1}(1+\gamma),\ldots, \tilde L_1\sum_{i=0}^{n-1} \gamma^i, M\sum_{i=0}^{n} \gamma^i \right\}, \label{E:Ubar} \\
    \bar{L}'_n &=\min \textstyle \left\{\tilde{U}_n, \tilde{U}_{n-1}(1+\gamma),\ldots, \tilde{U}_1 \sum_{i=0}^{n-1} \gamma^i, -M \sum_{i=0}^{n} \gamma^i \right\},\label{E:Lbar'} \\
    \bar{U}'_n &= \textstyle \max \left\{\tilde L_n, \tilde L_{n-1}(1+\gamma),\ldots, \tilde L_1\sum_{i=0}^{n-1} \gamma^i, M\sum_{i=0}^{n} \gamma^i \right\}, \label{E:Ubar'} 
\end{align}
and $ M $ is a finite positive scalar defined as
$ M= \max\left\{R_{\max}, \; \max_{(s,a)} Q_0(s,a)\right\}.$

The result follows from the claim in \eqref{C:LQ'U}.
To see this note that at any iteration $n$, $\bar{L}_n$ and $\bar{L}'_n$ are bounded below by $-\rho \sum_{i=0}^{n}\gamma^i$ since each term inside the minimum of \eqref{E:Lbar} and \eqref{E:Lbar'} is bounded below by $-\rho \sum_{i=0}^{n}\gamma^i$. 
As $n \rightarrow \infty$, we have
\begin{equation}
\frac{-\rho}{1-\gamma} \le \liminf_{n \rightarrow \infty} \bar{L}_n \quad \text{and} \quad \frac{-\rho}{1-\gamma} \le \liminf_{n \rightarrow \infty} \bar{L}'_n. 
\label{E:LIMINFLBAR}
\end{equation}
An analogous argument 
yields
\begin{equation}
\limsup_{n \rightarrow \infty} \bar{U}_n \le \frac{\rho}{1-\gamma} \quad \text{and} \quad \limsup_{n \rightarrow \infty} \bar{U}'_n \le \frac{\rho}{1-\gamma}. 
\label{E:LIMSUPUBAR}
\end{equation}
Boundedness of $Q_n(s,a)$ and $Q'_n(s,a)$ for all $(s,a) \in \mathcal S \times \mathcal A$ follows from \eqref{C:LQ'U}, \eqref{E:LIMINFLBAR} and \eqref{E:LIMSUPUBAR}.

\noindent Now, we prove our claim in \eqref{C:LQ'U} by induction.
Since Algorithm \ref{A:IV-LBQL} is asynchronous,  at the $n$th iteration, the updates for the action-value iterates for $(s,a)$, $Q_{n+1}(s,a)$ and $Q'_{n+1}(s,a)$, are either according to \eqref{Qupdate} and \eqref{Q'update} (case 1) or set equal to $Q_{n}(s,a)$ and $Q'_{n}(s,a)$ respectively (case 2).

We first focus on $Q'_n(s,a) \le \bar{U}_n$ part of \eqref{C:LQ'U}, since $\bar{L}'_n \le Q'_n(s,a)$ and $ \bar{L}_n \le Q_n(s,a) \le \bar{U}_n$ proceed in an analogous manner. For $n=1$, we have $Q'_0(s,a)= Q_0(s,a)$, so if the update is carried out as in case 1, 
\begin{equation*}
\begin{aligned}
  Q_1(s,a) &= \textstyle (1- \alpha_0(s,a)) \, Q'_0(s,a) + \alpha_0(s,a) \, [r(s,a) + \gamma \max_a Q'_0(s',a)] \\
  &\le (1- \alpha_0(s,a)) M + \alpha_0(s,a) M + \alpha_0(s,a) \gamma M \\
  &\le M (1+ \gamma) 
  \end{aligned}
\end{equation*}
so $Q'_1(s,a) \le \max \{ L_1(s,a), \min \{ U_1(s,a), M (1+ \gamma)  \} \}$.
Now consider the case where $U_1(s,a) \le M(1+ \gamma)$. Since $Q_0$ is bounded by $\rho$ and $L_0(s,a) \le U_0(s,a)$ then by Lemma \ref{L:LQU1}, we have $L_1(s,a) \le U_1(s,a)$, so
\begin{equation}
Q'_1(s,a) \le U_1(s,a) \le M(1+\gamma).
\label{E:LI1}
\end{equation}
Otherwise, if $U_1(s,a) \ge M(1+ \gamma)$, we then have 
\begin{equation}
Q'_1(s,a) \le \max \{ L_1(s,a), M(1+\gamma) \}.
\label{E:LI2}
\end{equation}
From \eqref{E:LI1} and \eqref{E:LI2}, we have 
\begin{equation}
\begin{split}
    Q'_1(s,a) &\le \max \{ L_1(s,a), M(1+\gamma) \} \\&
    \le \max \{\tilde L_1, M(1+\gamma) \}.
    \end{split}
\end{equation} 
If the update is carried out as in case 2, we have,
\begin{equation*}
\begin{split}
    Q'_1(s,a) &= Q'_0(s,a) \\ &
    \le M \\ &
    < M(1 + \gamma) \\ &
    \le \max \{\tilde L_1, M(1+\gamma) \}.
\end{split}
\end{equation*}
Thus $\bar{U}'_n(s,a)$ part of 
\eqref{C:LQ'U} is true for $n=1$. Suppose that it 
is true for $n= 1,2,\ldots, k$. We will show it for $n=k+1$. Consider first the instance where the update is carried out according to \mbox{case 1}. We do casework on the inequality
\begin{equation}
    \textstyle Q'_k(s,a) \le \max \left \{\tilde L_k,\tilde L_{k-1}(1+\gamma),\ldots, \tilde L_1\sum_{i=0}^{k-1} \gamma^i, M\sum_{i=0}^{k} \gamma^i \right\},
    \label{E:LI3}
\end{equation}
which holds for all $(s,a)$.
First, let us consider the case where the right-hand-side of \eqref{E:LI3} is equal to $\tilde L_{k'}\sum_{i=0}^{k-k'} \gamma^i$ for some $k'$ such that $1\le k' \le k$. Then, we have
\begin{equation}
\begin{split}
    Q_{k+1}(s,a) &= \textstyle (1- \alpha_k(s,a)) Q'_k(s,a) + \alpha_k(s,a) [r(s,a) + \gamma \max_a Q'_k(s',a)] \\ & 
    \le \textstyle (1- \alpha_k(s,a)) \tilde L_{k'}\sum_{i=0}^{k-k'} \gamma^i + \alpha_k(s,a) M + \alpha_k(s,a) \gamma \tilde L_{k'}\sum_{i=0}^{k-k'} \gamma^i \\ &
    \le \textstyle  (1- \alpha_k(s,a)) \tilde L_{k'} \sum_{i=0}^{k-k'} \gamma^i  + \alpha_k(s,a) \tilde L_{k'}  +  \alpha_k(s,a) \tilde L_{k'}\sum_{i=1}^{k-k'+1} \gamma^i \\ &
    = \textstyle  (1- \alpha_k(s,a)) \tilde L_{k'}\sum_{i=0}^{k-k'} \gamma^i  + \alpha_k(s,a) \tilde L_{k'} \sum_{i=0}^{k-k'} \gamma^i   +  \alpha_k(s,a) \tilde L_{k'} \gamma^{k-k'+1}\\ & 
    \le \textstyle \tilde L_{k'}\sum_{i=0}^{k-k'}  \gamma^i + \tilde L_{k'} \gamma^{k-k'+1} \\ &
    = \textstyle \tilde L_{k'}\sum_{i=0}^{k-k'+1} \gamma^i 
    \label{E:LI4}
    \end{split}
\end{equation}
The first inequality holds by the induction assumption \eqref{E:LI3}. The second inequality holds since in this case we have the right-hand-side of \eqref{E:LI3} is equal to $\tilde L_{k'}(1+ \gamma + \ldots+ \gamma^{k-k'})$. It follows that 
\[
\tilde L_{k'}(1+ \gamma + \ldots+ \gamma^{k-k'}) \ge M(1+\gamma+\ldots+ \gamma^{k}),
\]
which implies that $\tilde L_{k'}\ge M$. Finally, the third inequality holds by the assumption that $\alpha_n(s,a) \le 1$ for all $n$.
We have 
\begin{equation*}
\begin{split}
    Q'_{k+1}(s,a) &=\max \{L_{k+1}(s,a), \min \{U_{k+1}(s,a), Q_{k+1}(s,a) \} \} \\ &
\le \max\{ L_{k+1}(s,a), \min \{ U_{k+1}(s,a), \tilde L_{k'}(1+ \gamma + \ldots+ \gamma^{k-k'+1}) \} \}.
\end{split}
\end{equation*}
Now, consider the case where $U_{k+1}(s,a) \le \tilde L_{k'}(1+ \gamma + \ldots+ \gamma^{k-k'+1}) $. By the induction assumption, 
$Q'_n(s,a)$ is bounded below by $-\rho \sum_{i=0}^{n}\gamma^i$ and above by $\rho \sum_{i=0}^{n}\gamma^i$ for all $(s,a) \in \mathcal S \times \mathcal A$ and all $n= 1,2,\ldots, k$. Since $L_0(s,a) \le U_0(s,a)$,  Lemma \ref{L:LQU1} can be applied iteratively on $n=1,\ldots,k+1$ to obtain that $L_{K+1}(s,a) \le U_{k+1}(s,a)$ for all $(s,a) \in \mathcal S \times \mathcal A$. Thus, we have
\begin{equation}
 Q'_{k+1}(s,a) \le U_{k+1}(s,a) \le \tilde L_{k'}(1+ \gamma + \ldots+ \gamma^{k-k'+1}) .
\label{E:LI-}
\end{equation}
Otherwise, if $U_{k+1}(s,a) \ge \tilde L_{k'}(1+ \gamma + \ldots+ \gamma^{k-k'+1})$, we have 
\begin{align}
Q'_{k+1}(s,a) &\le \max \{ L_{k+1}(s,a), \tilde L_{k'}(1+ \gamma + \ldots+ \gamma^{k-k'+1}) \} \nonumber \\
&\le \max\{ \tilde L_{k+1}, \tilde L_{k'}(1+ \gamma + \ldots+ \gamma^{k-k'+1})\}. \label{E:LI--}
\end{align}
Moving on to the case where the right-hand-side of \eqref{E:LI3} is equal to $M(1+\gamma+\ldots+ \gamma^{k})$:
\begin{equation}
\begin{split}
   \textstyle Q_{k+1}(s,a) &=(1- \alpha_k(s,a)) \, Q'_k(s,a) + \alpha_k(s,a)\, [r(s,a) + \gamma \max_a Q'_k(s',a)] \\ & 
   \textstyle \le (1- \alpha_k(s,a))M\sum_{i=0}^{k} \gamma^i + \alpha_k(s,a) M  + \alpha_k(s,a) \gamma  M\sum_{i=0}^{k} \gamma^i \\ &
   \textstyle = (1- \alpha_k(s,a))M\sum_{i=0}^{k} \gamma^i + \alpha_k(s,a) M  +  \alpha_k(s,a) M\sum_{i=1}^{k+1} \gamma^i\\ &
  \textstyle  \le M \sum_{i=0}^{k} \gamma^i - \alpha_k(s,a) M\sum_{i=0}^{k} \gamma^i 
     + \alpha_k(s,a) M \sum_{i=0}^{k} \gamma^i+  M \gamma^{k+1}\\ & 
  \textstyle  = M(1+ \gamma + \ldots+ \gamma^{k+1}).
    \label{E:LI5}
    \end{split}
\end{equation}
We have 
\begin{equation*}
\begin{split}
    Q'_{k+1}(s,a) &=\max\{L_{k+1}(s,a), \min(U_{k+1}(s,a), Q_{k+1}(s,a))\} \\ &
\le \max\{L_{k+1}(s,a), \min(U_{k+1}(s,a), M(1+ \gamma + \ldots+ \gamma^{k+1}) \}\}.
\end{split}
\end{equation*}
Now if $U_{k+1}(s,a) \le M(1+ \gamma + \ldots+ \gamma^{k+1})$, then by applying Lemma \ref{L:LQU1} as before,
\begin{equation}
Q'_{k+1}(s,a) \le U_{k+1}(s,a) \le M(1+ \gamma + \ldots+ \gamma^{k+1}) .
\label{E:LI*}
\end{equation}
Otherwise, if $U_{k+1}(s,a) \ge M(1+ \gamma + \ldots+ \gamma^{k+1})$, we have 
\begin{align}
Q'_{k+1}(s,a) &\le \max\{L_{k+1}(s,a),M(1+ \gamma + \ldots+ \gamma^{k+1})\} \nonumber \\
 &\le \max\{\tilde L_{k+1},M(1+ \gamma + \ldots+ \gamma^{k+1})\}. \label{E:LI**}
\end{align}
Now, if the update is carried out according to case 2,
\begin{equation}
\begin{split}
    Q'_{k+1}(s,a) &= Q'_K(s,a) \\ &
    \textstyle \le \max \{\tilde L_k,\tilde L_{k-1}(1+\gamma),\ldots, \tilde L_1\sum_{i=0}^{k-1} \gamma^i, M\sum_{i=0}^{k} \gamma^i \} \\ &
    \textstyle \le \max  \{\tilde L_{k+1}, (1+\gamma)\tilde L_k,\ldots, \tilde L_1\sum_{i=0}^{k} \gamma^i, \textstyle M\sum_{i=0}^{k+1} \gamma^i \}.
\end{split}
\label{E:LI***}
\end{equation}
By 
\eqref{E:LI-}, \eqref{E:LI--}, \eqref{E:LI*}, \eqref{E:LI**} and \eqref{E:LI***}, we have
$Q'_{k+1}(s,a) \le \bar{U'}_{k+1}$. 
A similar argument can be made to show $\bar{L}_n \le Q'_n(s,a)$ and $\bar{L}_n \le Q_n(s,a) \le \bar{U}_n$, which completes the inductive proof.
\end{proof}

\begin{proof}
Part \ref{L:LQU2}:
Fix an $(s,a) \in \mathcal S \times \mathcal A$.
By part \ref{L:BQ'} we have the action-value iterates $Q_n$ and $Q'_n$ are bounded for all $n$. 
We denote the ``sampling noise'' term using $$\xi^L_n(s,a)=\hat Q_{n,0}^L(s,a) - \E[\hat Q_{n,0}^L(s,a)].$$ We also define an accumulated noise process started at iteration $\nu$ by $W^L_{\nu,\nu}(s,a) = 0$, and
\begin{align*}
W^L_{n+1,\nu}(s,a) = \left(1- \alpha_n (s,a)\right) W^L_{n,\nu}(s,a)+ \alpha_n(s,a) \, \xi^L_{n+1}(s,a) \quad \forall \, \, n \ge \nu,
\end{align*}
which averages noise terms together across iterations.  Note that $\tau$ is an almost surely finite stopping time, the rewards $r(s,a)$ are uniformly bounded, and $Q_{n+1}$ is also bounded (by part \ref{L:BQ'}). Then, $\hat Q_{n,0}^L$ is bounded by some random variable and so is the conditional variance of $\xi^L_{n}(s,a)$.
Hence, Corollary 4.1 in \cite{bertsekas1996neuro} applies and it follows that 
\begin{align*}
\lim_{n \rightarrow \infty} W^L_{n,\nu} (s,a) = 0 \quad \forall \, \nu \ge 0.
\end{align*}
Let $\tilde{\nu}$ be large enough so that $\alpha_n(s,a) \le 1$ for all $n \ge \tilde{\nu}$.
We also define
\begin{align*}
Y^L_{\tilde{\nu}}(s,a) &= \rho, \nonumber \\
Y^L_{n+1}(s,a) &= (1 - \alpha_{n}(s,a)) \, Y^L_{n}(s,a) + \alpha_{n}(s,a) \, Q^*(s,a), \quad \forall \, n \ge \tilde{\nu}.
\label{E:YL}
\end{align*}

It is easy to see that the sequence $Y^L_n(s,a) \rightarrow Q^*(s,a)$.
We claim that for all iterations $n \ge \tilde{\nu}$, it holds that
$$ L_n(s,a) \le \min \{ \rho,\, Y^L_n(s,a) + W^L_{n, \tilde{\nu}}(s,a)\}.$$ 
To prove this claim, we proceed by induction on $n$. For $n = \tilde{\nu}$, we have 
\[
Y^L_{\tilde{\nu}}(s,a) = \rho \quad \text{and} \quad W^L_{\tilde{\nu},\tilde{\nu}}(s,a) = 0,
\] 
so it is clear that the statement is true for the base case. We now show that it is true for $n+1$ given that it holds at $n$:
\begin{equation*}
    \begin{split}
        L_{n+1}(s,a) &= \min \{\rho, \,  (1 - \alpha_n(s,a)) \, L_n(s,a) + \alpha_n(s,a) \, (\hat Q_{n,0}^L(s,a) - \E[ \hat Q_{n,0}^L(s,a)] + \E[\hat Q_{n,0}^L(s,a)] ) \}  \\ &
         = \min \{\rho, \, (1 - \alpha_n(s,a)) \, L_n(s,a) + \alpha_n(s,a) \, \xi^L_n(s,a) + \alpha_n(s,a) \, \E[\hat Q_{n,0}^L(s,a)] \} \\ &
         \le \min \{\rho, \, (1- \alpha_n(s,a)) \, (Y^L_n(s,a) + W^L_{n, \nu_k}(s,a)) + \alpha_n(s,a) \, \xi^L_n(s,a) + \alpha_n(s,a) \, Q^*(s,a) \}\\ &
         \le \min \{\rho, \,  Y^L_{n+1}(s,a) + W^L_{n+1, \tilde{\nu}}(s,a) \},
    \end{split}
\end{equation*}
where the first inequality follows by the induction hypothesis and $\E[\hat Q_{n,0}^L(s,a)] \le Q^*(s,a)$ follows by Proposition \ref{P:EmpDR}. Next, since $Y^L_n(s,a) \rightarrow Q^*(s,a)$, $ W^L_{n, \nu_k}(s,a)\rightarrow 0$ and $Q^*(s,a) \le \rho$, we have 
\[
\limsup_{n \rightarrow \infty} \,   L_{n}(s,a)  \le  Q^*(s,a).
\]
Therefore, since our choice of $(s,a)$ was arbitrary, it follows that for every $\eta >0$, there exists some time $n'$ such that $L_n(s,a) \le Q^*(s,a) + \eta$ for all $(s,a) \in \mathcal{S} \times \mathcal{A}$ and $n \ge n'$. 

Using Proposition \ref{P:EmpDR}, $Q^*(s,a) \le \E[\hat Q_{n,0}^U(s,a)]$, a similar argument as the above can be used to establish that 
\[
Q^*(s,a) \le \liminf_{n \rightarrow \infty} \,   U_{n}(s,a).
\]
Hence, there exists some time $n''$ such that $Q^*(s,a) - \eta \le U_n(s,a)$ for all $(s,a)$ and $n \ge n''$. Take $n_0$ to be some time greater than $n'$ and $n''$ and the result follows. 
\end{proof}

\subsection{Proof of Lemma \ref{L:LU} }
\label{Appendix:LU}
\begin{proof}
We use induction on $n$.
Since for all $(s,a) \in \mathcal S \times \mathcal A$,  $L_0(s,a) \le U_0(s,a)$ and $-\rho \le Q'_0(s,a) \le \rho$ then $L_1(s,a) \le U_1(s,a)$
by Lemma \ref{L:LQU1}.
Suppose that $L_n(s,a) \le U_n(s,a)$ holds for all $(s,a)$ for all $n=1,\ldots, k$. For all $(s,a) \in \mathcal S \times \mathcal A$, we have $Q'_k(s,a)$ is bounded since by Proposition \ref{L:LQU}\ref{L:BQ'} $Q'_n(s,a)$ is bounded for all $n$. We also have $L_k(s,a) \le U_k(s,a)$ for all $(s,a)$ by the induction assumption. Applying Lemma \ref{L:LQU1} again at $n=k+1$ yields $L_{k+1}(s,a) \le U_{k+1}(s,a)$ and the inductive proof is complete. 
\end{proof}

\subsection{Proof of Theorem \ref{T:Conv}}
\begin{proof}
We first prove part \ref{T:1}. 
We start by writing Algorithm \ref{A:IV-LBQL} using DP operator notation. Define a mapping $H$ such that
\begin{equation*}
\textstyle (HQ')(s,a) = r(s,a) + \gamma \E \left[\max_{a'} Q'(s',a')\right],    
\end{equation*}
where $s' = f(s,a,w)$.
It is well-known that the mapping $H$ is a $\gamma$-contraction in the maximum norm. We also define a random noise term
\begin{equation}
\textstyle
    \xi_n(s,a) =  \gamma \max_{a' }Q'_n(s',a') - \gamma \E \left[\max_{a'} Q'_n(s',a')\right]  .
    \label{E:ND}
\end{equation}
 The main update rules of Algorithm \ref{A:IV-LBQL}  can then be written as
\begin{align}
        & Q_{n+1}(s,a)  =  (1 - \alpha_n(s,a)) \, Q'_n(s,a) + \alpha_n(s,a) \, \left[(HQ'_{n})(s,a) + \xi_{n+1}(s,a) \right], 
        \nonumber \\ & 
        U_{n+1}(s,a) = \Pi_{[-\rho, \, \infty]} \left [ (1 - \beta_n(s,a)) \, U_n(s,a) + \beta_n(s,a) \, \hat Q^{U}_{0}(s,a) \right ],  \nonumber \\ &
        L_{n+1}(s,a) =  \Pi_{[\infty, \, \rho]} \left [(1 - \beta_n(s,a)) \, L_n(s,a) + \beta_n(s,a) \, \hat Q^L_{0}(s,a)\} \right ],   \nonumber \\ &
        Q'_{n+1}(s,a) =\Pi_{[ L_{n+1}(s,a),\, U_{n+1}(s,a)]}  \left [  Q_{n+1}(s,a)\right ]. \label{E:**} &
\end{align}
Assume without loss of generality that $Q^*(s,a)=0$ for all state-action pairs $(s,a)$. This can be established by shifting the origin of the coordinate system. Note that by \eqref{E:**} at any iteration $n$ and for all $(s,a)$, we have $L_n(s,a) \le Q'_n(s,a) \le U_n(s,a)$.

We proceed via induction. First, note that by Propostion \ref{L:LQU}\ref{L:BQ'} the iterates of Algorithm \ref{A:IV-LBQL} $Q'_n(s,a)$ are bounded in the sense that there exists a constant $D_0$ such that $| Q'_n(s,a)| \le D_0$ for all $(s,a)$ and iterations $n$.
Define the sequence $D_{k+1}=(\gamma +  \epsilon) \, D_k$, such that $\gamma + \epsilon < 1$ and $\epsilon > 0$. Clearly, $D_k\rightarrow 0$.
Suppose that there exists some time $n_k$ such that for all $(s,a)$,
$$ \max\{-D_k , L_n(s,a) \} \le Q'_n(s,a) \le \min\{D_k, U_n(s,a)\}, \ \  \forall n \ge n_k.$$ 
We will show that this implies the existence of some time $n_{k+1}$ such that 
$$ \max\{-D_{k+1} , L_{n}(s,a) \} \le Q'_n(s,a) \le \min\{D_{k+1}, U_{n}(s,a)\} \ \ \forall \, (s,a), \, n \ge n_{k+1}. $$
This implies that $Q'_n(s,a)$ converges to $Q^*(s,a)=0$ for all $(s,a)$. We also assume that $\alpha_n(s,a) \le 1$ for all $(s,a)$ and $n$. Define an accumulated noise process started at $n_k$ by $W_{n_k,n_k}(s,a) = 0$, and 
\begin{align}
W_{n+1,n_k}(s,a) = \left(1- \alpha_n (s,a)\right) W_{n,n_k}(s,a)+ \alpha_n(s,a) \, \xi_{n+1}(s,a),\quad \forall \, n \ge n_k,
\label{E:Wnnk}
\end{align}
where $\xi_n(s,a)$ is as defined in \eqref{E:ND}. We now use Corollary 4.1 in \cite{bertsekas1996neuro} which states that under the assumptions of Theorem \ref{T:Conv} on the step size $\alpha_n(s,a)$, and if $\E[\xi_n(s,a) \, | \,\mathcal{F}_n]=0$ and $\E[\xi^2_n(s,a) \,|\,\mathcal{F}_n] \le A_n$, where the random variable $A_n$ is bounded with probability 1, the sequence $W_{n+1,n_k}(s,a)$ defined in \eqref{E:Wnnk} converges to zero, with probability 1.
From our definition of the stochastic approximation noise $\xi_n(s,a)$ in \eqref{E:ND}, we have 
\[
\textstyle \E[\xi_n(s,a) \, | \,\mathcal{F}_n]=0 \quad \text{and} \quad \E[\xi^2_n(s,a) \,|\,\mathcal{F}_n] \le C (1 + \max_{s',a'} Q_n^{'2}(s',a')),\] where $C$ is a constant. 
Then, it follows that 
\begin{align*}
\lim_{n\rightarrow \infty} W_{n,n_k} (s,a) = 0 \quad \forall \, (s,a), \; n_k.
\end{align*}
Now, for the sake of completeness, we restate a lemma from \cite{bertsekas1996neuro} below, which we will use to bound the accumulated noise.
\begin{lem}[Lemma 4.2 in \cite{bertsekas1996neuro}]
For every $\delta >0$, and with probability one, there exists some $n'$ such that $|W_{n,n'}(s,a)|\le \delta$, for all $n \ge n'$.
\label{L:L4.2Neuro}
\end{lem}
Using the above lemma, let $n_{k' }\ge n_k$ such that, for all $n \ge n_{k'}$ we have 
\[
|W_{n,n_{k'}}(s,a)|\le \gamma \epsilon D_k < \gamma D_k.
\]
Furthermore, by Proposition \ref{L:LQU}\ref{L:LQU2} let $\nu_{k }\ge n_{k'}$ such that, for all $n \ge \nu_{k}$ we have 
$$L_n(s,a) \le \gamma D_k - \gamma \epsilon D_k \quad \text{and} \quad \gamma \epsilon D_k - \gamma D_k \le U_n(s,a).$$
Define another sequence $Y_n$ that starts at iteration $\nu_k$.
\begin{align}
Y_{\nu_k}(s,a) = D_{k} \quad \text{and} \quad  Y_{n+1}(s,a) = (1 - \alpha_{n}(s,a)) \, Y_{n}(s,a) + \alpha_{n}(s,a) \, \gamma \,D_{k}
\label{E:Y}
\end{align}
Note that it is easy to show that the sequence $Y_n(s,a)$ in \eqref{E:Y} is decreasing, bounded below by $\gamma D_k$, and converges to $\gamma D_k$ as $n \rightarrow \infty$. Now we state the following lemma.
\begin{lem} For all state-action pairs $(s,a)$ and iterations $n \ge \nu_k$, it holds that:
\begin{enumerate}[label={(\arabic*)}]
    \item $Q'_n(s,a) \le \min\{U_n(s,a),Y_n(s,a) + W_{n,\nu_k}(s,a)\}$, \label{L:TL1}
    \item $\max\{L_n(s,a),-Y_n(s,a) + W_{n,\nu_k}(s,a)\} \le Q'_n(s,a)$.  \label{L:TL2}
\end{enumerate}
\end{lem}
\begin{proof}
We focus on part \ref{L:TL1}. For the base case $n = \nu_k$, the statement holds because $Y_{\nu_k}(s,a) = D_k$ and $W_{\nu_k,\nu_k}(s,a) = 0$. We assume it is true for $n$ and show that it continues to hold for $n+1$:
\begin{align*}
Q_{n+1}(s,a) &= (1 - \alpha_n(s,a)) Q'_n(s,a) + \alpha_n(s,a)\, \left[(HQ'_{n})(s,a) + \xi_{n+1}(s,a) \right]\\
 &\hspace{3pt}\begin{aligned}[t]\le (1 - \alpha_n(s,a))  &\min\{U_n(s,a),Y_n(s,a) + W_{n,\nu_k}(s,a)\} \\
 & + \alpha_n(s,a) \, (HQ'_{n})(s,a) + \alpha_n(s,a) \, \xi_{n+1}(s,a)\end{aligned}\\
 &\le (1 - \alpha_n(s,a)) \, (Y_n(s,a) + W_{n,\nu_k}(s,a)) + \alpha_n(s,a) \, \gamma D_k + \alpha_n(s,a) \, \xi_{n+1}(s,a)\\
 &\le Y_{n+1}(s,a) + W_{n+1,\nu_k}(s,a),
\end{align*}
where we used  $(HQ'_{n}) \le \gamma \Vert Q'_{n} \Vert \le \gamma D_k$. Now, we have
\begin{align*}
Q'_{n+1}(s,a) &= \Pi_{[ L_{n+1}(s,a),\, U_{n+1}(s,a)]}  \left [  Q_{n+1}(s,a)\right ] \\ &
\le  \Pi_{[ L_{n+1}(s,a),\, U_{n+1}(s,a)]}  \left [ Y_{n+1}(s,a) + W_{n+1,\nu_k}(s,a) \right ]  \\ 
&\le  \min \{U_{n+1}(s,a), Y_{n+1}(s,a) + W_{n+1,\nu_k}(s,a) \}.
\end{align*}
The first inequality holds because 
\[Q_{n+1}(s,a) \le Y_{n+1}(s,a) + W_{n+1,\nu_k}(s,a).
\]
The second inequality holds because 
$Y_{n+1}(s,a) +W_{n+1,\nu_k}(s,a) \ge \gamma D_k - \gamma \epsilon D_k $, $ L_n(s,a) \le \gamma D_k - \gamma \epsilon D_k$, and $L_n(s,a) \le U_n(s,a)$ by Lemma \ref{L:LU}. Symmetrically, it can be shown that
\begin{align*}
\max \{ L_{n+1}(s,a), -Y_{n+1}(s,a) + W_{n+1,\nu_k}(s,a) \} \le Q'_{n+1}(s,a),
\end{align*}	
which completes the proof. 
\end{proof}
\noindent
Since $Y_{n}(s,a) \rightarrow \gamma  D_{k}$ and $W_{n,\nu_k}(s,a) \rightarrow 0$, we have
\[
{\limsup}_{n\rightarrow \infty} \Vert Q'_{n} \Vert \le \gamma D_{k} < D_{k+1}.
\]
Therefore, there exists some time $n_{k+1}$ such that $$ \max\{-D_{k+1} , L_{n}(s,a) \} \le Q'_n(s,a) \le \min\{D_{k+1}, U_{n}(s,a)\} \ \ \forall \, (s,a), \, n \ge n_{k+1},$$ 
completing thus the induction.

For part \ref{T:2} of the theorem: 
we fix $(s,a)$ and focus on the convergence analysis of $U_n(s,a)$ to $Q^*(s,a)$. A similar analysis can be done to show $L_n(s,a) \rightarrow Q^*(s,a)$ almost surely. First note that we can write \eqref{Uupdate} the update equation of $U_n(s,a)$ as:
\[ 
U_{n+1}(s_n,a_n) = \Pi_{[-\rho, \, \infty]} \left [ U_n(s_n,a_n) + \beta_n(s_n,a_n) \, \bigl[ \psi_n( U_n(s_n,a_n),Q_{n+1}(s_n,a_n)) \bigr] \right ]
\]
where $\psi_n(U_n(s_n,a_n),Q_{n+1}(s_n,a_n)) $ is the stochastic gradient and in this case is equal to $\hat Q^{U}_{0}(s_n,a_n) - U_n(s_n,a_n)$.
We define the noise terms 
\begin{align}
        &\bar \epsilon_{n+1}(s_n,a_n) = \psi_n(U_n(s_n,a_n), Q^*(s_n,a_n)) - \E[\psi_n(U_n(s_n,a_n), Q^*(s_n,a_n))] \\ &
    \bar \varepsilon_{n+1}(s_n,a_n) = \psi_n(U_n(s_n,a_n), Q_{n+1}(s_n,a_n)) - \psi_n(U_n(s_n,a_n), Q^*(s_n,a_n)).
\end{align}
Note here that $\bar \epsilon_{n+1}(s_n,a_n)$ represents the error that the sample gradient deviates from its mean when computed using the optimal action-value $Q^*$ and $\bar \varepsilon_{n+1}(s_n,a_n)$ is the error between the two evaluations of $\psi_n$ due only to the difference between $Q_{n+1}(s_n,a_n)$ and $Q^*(s_n,a_n)$. Thus, we have 
$$
\psi_n(U_n(s_n,a_n), Q_{n+1}(s_n,a_n)) = \E[\psi_n(U_n(s_n,a_n), Q^*(s_n,a_n))] + \bar \epsilon_{n+1}(s_n,a_n) + \bar \varepsilon_{n+1}(s_n,a_n).
$$
Since $Q_n \rightarrow Q^*$ by part \ref{T:1} of the Theorem, then $\bar \varepsilon_n(s_n,a_n) \rightarrow 0$ almost surely. It is now convenient to view $U_n(s,a)$ as a stochastic process in $n$, adapted to the filtration $\{ \mathcal F_n \}_{n \ge 0}.$ By definition of $\bar \epsilon_{n+1}(s,a)$, we have that
$$
\E[\bar \epsilon_{n+1}(s,a)|\mathcal F_{n}]=0 \quad a.s.
$$
Since $\bar \epsilon_{n+1}(s,a)$ is unbiased and $\bar \varepsilon_{n+1}(s,a)$ converges to zero, we can apply Theorem 2.4 of \cite{kushner2003stochastic}, a standard stochastic approximation convergence result, to conclude that $U_n(s,a) \rightarrow Q^*(s,a)$ almost surely. Since our choice of $(s,a)$ was arbitrary, this convergence holds for all $(s,a) \in \mathcal S \times \mathcal A.$
\end{proof}
\subsection{Proof of Lemma \ref{L:pLQL2}}
\begin{proof}
First note that Lemma \ref{L:LQU1} still holds in this case.
To see this, note that if the requirements of the lemma are satisfied (i.e., if  we are at iteration $n=1,2,\ldots,$ and in the previous iteration, we had $L_{n-1}(s,a) \le U_{n-1}(s,a)$ for all $(s,a)$ and $Q'_{n-1} \in \mathcal Q$), then $Q_n$ is bounded using the same argument as before. Since the rewards $r(s,a)$ are uniformly bounded and $\tau$ is an almost surely finite stopping time, then $\tilde Q^L_{n,0}$ and $\tilde Q^U_{n,0}$ are finite.
Moreover, since  $\tilde Q^L_{n,0}$ and $\tilde Q^U_{n,0}$ are computed using the same sample path $\mathbf{w}$,
it follows that 
\[
\tilde Q^U_{n,0}(s,a) - \tilde Q^L_{n,0}(s,a) \ge 0, \quad \text{ for all } (s,a) \in \mathcal S \times \mathcal A.
\]
This can be easily seen if we subtract \eqref{E:EmpQL_biased} from \eqref{E:EmpQG_biased}. Notice that the reward and the penalty will both cancel out and we have $\tilde Q^U_{n,t} - \tilde Q^L_{n,t} \ge 0$ for all $t=0,1,\ldots,\tau-1$.
With $L_{n-1}(s,a) \le U_{n-1}(s,a)$ and $\tilde Q^U_{n,0}(s,a) \ge \tilde Q^L_{n,0} (s,a)$ it follows that $L_{n} (s,a)\le U_{n}(s,a)$ for all $(s,a)$.

Now, we prove Proposition \ref{L:LQU} again for the experience replay buffer case.

For part \ref{L:BQ'}: our original proof still holds since Lemma \ref{L:LQU1} still holds.

For part \ref{L:LQU2}:
first note that since the experience replay buffer is updated with a new observation of the noise at every iteration, then by Borel's law of large numbers, we have our probability estimate $\hat p(w)$ for the noise converges to the true noise distribution $p(w)$ as $n\rightarrow \infty$, i.e.,
\begin{equation}
    \lim_{n\rightarrow \infty} \hat p_n(w) = p(w) \text{ for all } w \in \mathcal{W}. 
\end{equation}

\noindent Fix an $(s,a) \in \mathcal S \times \mathcal A$.
By part \ref{L:BQ'} we have the action-value iterates $Q_n$ and $Q'_n$ are bounded for\mbox{ all $n$.} 
Now we write the iterate $\tilde Q_{n,0}^L(s,a)$ in terms of a noise term and a bias term as follows,
$$ \tilde Q_{n,0}^L(s,a) =  \underbrace{\tilde Q_{n,0}^L(s,a)  - \E[\tilde Q_{n,0}^L(s,a)]}_{\rm noise} + \underbrace{\E [\tilde Q_{n,0}^L(s,a)] -   \E[ Q_{n,0}^L(s,a)]}_{\rm bias} + \E[ Q_{n,0}^L(s,a)].$$
Now, we define the noise term using $$\xi^L_n(s,a)=\tilde Q_{n,0}^L(s,a)  - \E[\tilde Q_{n,0}^L(s,a)].$$
Also, similar to the original proof we define an accumulated noise process started at iteration $\nu$ by $W^L_{\nu,\nu}(s,a) = 0$, and
\begin{align*}
W^L_{n+1,\nu}(s,a) = \left(1- \alpha_n (s,a)\right) W^L_{n,\nu}(s,a)+ \alpha_n(s,a) \, \xi^L_{n+1}(s,a) \quad \forall \, \, n \ge \nu,
\end{align*}
which averages noise terms together across iterations. 
We have $\E [\tilde Q_{n,0}^L(s,a)  - \E[\tilde Q_{n,0}^L(s,a)]| \mathcal F_n]=0$, so Corollary 4.1 applies and it follows that
\begin{align*}
\lim_{n \rightarrow \infty} W^L_{n,\nu} (s,a) = 0 \quad \forall \, \nu \ge 0.
\end{align*}
Let $\tilde{\nu}$ be large enough so that $\alpha_n(s,a) \le 1$ for all $n \ge \tilde{\nu}$.
We denote the bias term by
$$
\chi_n (s,a) = \E [\tilde Q_{n,0}^L(s,a)] -   \E[ Q_{n,0}^L(s,a)].
$$
Since as $n \rightarrow \infty$, we have $\hat p_n(w) \rightarrow p(w)$, the bias due to sampling from the experience buffer $\chi_n (s,a) \rightarrow 0$.
Let $\eta >0$ and  $\bar \nu \ge \tilde \nu$ be such that $|\chi(s,a)|\le \frac{\eta}{2}$ for all $n \ge \bar \nu$ and all $(s,a)$.
We also define
\begin{align*}
Y^L_{\tilde{\nu}}(s,a) &= \rho, \nonumber \\
Y^L_{n+1}(s,a) &= (1 - \alpha_{n}(s,a)) \, Y^L_{n}(s,a) + \alpha_{n}(s,a) \, Q^*(s,a) + \alpha_{n}(s,a) \, \frac{\eta}{2}, \quad \forall \, n \ge \bar{\nu}.
\label{E:YL}
\end{align*}

It is easy to see that the sequence $Y^L_n(s,a) \rightarrow Q^*(s,a) + \frac{\eta}{2}$.
Now we show that the following claim holds. Claim: for all iterations $n \ge \bar{\nu}$, it holds that
$$ L_n(s,a) \le \min \{ \rho,\, Y^L_n(s,a) + W^L_{n, \bar{\nu}}(s,a)\}.$$ 
To prove this claim, we proceed by induction on $n$. For $n = \bar{\nu}$, we have 
\[
Y^L_{\bar{\nu}}(s,a) = \rho \quad \text{and} \quad W^L_{\bar{\nu},\bar{\nu}}(s,a) = 0,
\] 
so the statement is true for the base case. We now show that it is true for $n+1$ given that it holds at $n$:
\begin{equation*}
    \begin{split}
        L_{n+1}(s,a) &= \min \{\rho, \,  (1 - \alpha_n(s,a)) \, L_n(s,a) + \alpha_n(s,a) \, (\tilde Q_{n,0}^L(s,a)  -  \E[\tilde Q_{n,0}^L(s,a)] \\ & \qquad \quad + \E [\tilde Q_{n,0}^L(s,a)] -   \E[ Q_{n,0}^L(s,a)] + \E[ Q_{n,0}^L(s,a)] ) \}  \\ &
         = \min \{\rho, \, (1 - \alpha_n(s,a)) \, L_n(s,a) + \alpha_n(s,a) \, \xi^L_n(s,a) + \alpha_n(s,a) \,\chi_n(s,a) \\& 
        \qquad  \qquad  + \alpha_n(s,a) \, \E[ Q_{n,0}^L(s,a)]  \} \\ &
         \le \min \{\rho, \, (1- \alpha_n(s,a)) \, (Y^L_n(s,a) + W^L_{n, \nu_k}(s,a)) + \alpha_n(s,a) \, \xi^L_n(s,a)  \\ & \qquad \quad  + \alpha_n(s,a) \, \frac{\eta}{2} + \alpha_n(s,a) \, Q^*(s,a) \}\\ &
         \le \min \{\rho, \,  Y^L_{n+1}(s,a) + W^L_{n+1, \tilde{\nu}}(s,a) \},
    \end{split}
\end{equation*}
where the first inequality follows by the induction hypothesis and $\E[ Q_{n,0}^L(s,a)] \le Q^*(s,a)$. Next, since $Y^L_n(s,a) \rightarrow Q^*(s,a) + \frac{\eta}{2} $, $ W^L_{n, \nu_k}(s,a)\rightarrow 0$, then if  $Q^*(s,a) + \frac{\eta}{2} \le \rho$, we have 
\[
\limsup_{n \rightarrow \infty} \,   L_{n}(s,a)  \le  Q^*(s,a) + \frac{\eta}{2}.
\]
Otherwise, if $\rho < Q^*(s,a) + \frac{\eta}{2}$, then 
\[
\limsup_{n \rightarrow \infty} \,   L_{n}(s,a)  \le \rho < Q^*(s,a) + \frac{\eta}{2}.
\]
Therefore, since our choice of $(s,a)$ was arbitrary, it follows that for every $\eta >0$, there exists some time $n'$ such that $L_n(s,a) \le Q^*(s,a) + \eta$ for all $(s,a) \in \mathcal{S} \times \mathcal{A}$ and $n \ge n'$. 

Using Proposition \ref{P:ATDR}\ref{P:WD}, $Q^*(s,a) \le \E[ Q_{n,0}^U(s,a)]$, a similar argument as the above can be used to establish that 
\[
Q^*(s,a) - \frac{\eta}{2} \le \liminf_{n \rightarrow \infty} \,   U_{n}(s,a).
\]
Hence, there exists some time $n''$ such that $Q^*(s,a) - \eta \le U_n(s,a)$ for all $(s,a)$ and $n \ge n''$. Take $n_0$ to be some time greater than $n'$ and $n''$ and the result follows. 

The proof of Lemma \ref{L:LU} when using an experience buffer is similar to that given in appendix \ref{Appendix:LU} so it is omitted.
\end{proof}

\subsection{Proof of Theorem \ref{T:ConvERLBQL}}
\begin{proof}
The proof of both parts \ref{T:ER1} and \ref{T:ER2} are similar to that of Theorem \ref{T:Conv}, so we omit them.
\end{proof}

\clearpage

\section{LBQL with Experience Replay Algorithm}
\label{Appendix:LBQLwithExpRep}
\begin{algorithm}[htb]
    \caption{LBQL with Experience Replay}
    \label{A:ER-LBQL}
\begin{algorithmic}
  \STATE {\bfseries Input:} Initial estimates $L_0 \le Q_0 \le U_0$, batch size $K$, stepsize rules $\alpha_n(s,a)$, $\beta_n(s,a),$ and noise buffer $\mathcal B$.
  \STATE {\bfseries Output:}Approximations $\{L_n\}$, $\{Q'_n\}$, and $\{U_n\}$.
  \STATE  Set $Q_0' = Q_0$ and choose an initial state $s_0$.
  \FOR{$n = 0, 1, 2, \ldots$}
  \STATE Choose an action $a_n$ via some behavior policy (e.g., $\epsilon$-greedy). Observe $w_{n+1}$.
  \STATE Store $w_{n+1}$ in $\mathcal B$ and update $\hat p_n(w_{n+1})$.
  \STATE  Perform a standard $Q$-learning update:
  \vspace{-0.1em}
  \begin{equation*}
  \begin{aligned}
  Q_{n+1}(&s_n,a_n) = Q'_n(s_n,a_n) + \alpha_n(s_n, a_n) \Bigl[ r_n(s_n,a_n) + \gamma \max_a Q'_n(s_{n+1},a) - Q'_n(s_{n},a_n) \Bigr ].
  \end{aligned}
\end{equation*}
  \vspace{-1.2em}
  \STATE Sample randomly a sample path $\mathbf{w}=(w_1,w_2,\ldots, w_{\tau} )$ from $\mathcal B$, where $\tau \sim {\rm Geom}(1-\gamma)$.
  \STATE Set $\varphi = Q_{n+1}$. Using $\mathbf{w}$ and the current $\hat p_n$
  compute $\tilde Q^{U}_{0}(s_n,a_n)$ and $\tilde Q^L_{0}(s_n,a_n)$, using  \eqref{E:EmpQG_biased} and \eqref{E:EmpQL_biased}, respectively.\; \label{Qu_and_Ql_update}
  \STATE Update and enforce upper and lower bounds:
  \begin{align*}
  \begin{split}
  U_{n+1}&(s_n,a_n) = \Pi_{[-\rho, \, \infty]} \Bigl [ U_n(s_n,a_n)  + \beta_n(s_n,a_n) \, \bigl[ \tilde Q^{U}_{0}(s_n,a_n) - U_n(s_n,a_n) \bigr] \Bigr ],\nonumber
 \end{split}\\
 \begin{split}
   L_{n+1}&(s_n,a_n) = \Pi_{[\infty, \, \rho]} \Bigl [ L_n(s_n,a_n)  + \beta_n(s_n,a_n) \, \bigl[\tilde Q^L_{0}(s_n,a_n)- L_n(s_n,a_n) \bigr] \Bigr ],
  \nonumber
  \end{split}
  \end{align*}
  \vspace{-1em}
  \begin{equation*}
  \begin{split}
      Q'_{n+1}&(s_n,a_n) = \Pi_{[ L_{n+1}(s_n,a_n),\, U_{n+1}(s_n,a_n)]} \left [Q_{n+1}(s_n,a_n)\right]
  \end{split}
  \end{equation*}
\ENDFOR
\end{algorithmic}
\end{algorithm}
\clearpage

\section{Implementation Details of LBQL with Experience Replay}
\label{Appendix:LBQLwERSV}

     We use a noise buffer $\mathcal B$ of size $\kappa$ to record the noise values $w$ that have been previously observed. The buffer $\mathcal B$ is used to generate the sample path $\mathbf w$ and the batch sample $\{w_1, \ldots, w_K \}$ used to estimate the expectation in the penalty function. 
     Here, it is not necessary that the noise space $\mathcal W$ is finite. 
     This is also convenient in problems with a large noise support such as the car-sharing problem with four stations where we have two sources of noise. Specifically, the noise due to the distribution of the rentals among the stations has a very large support.
     
     In order to reduce the computational requirements of LBQL, the lower and upper bounds updates are done every $m$ steps and only if the difference between the current values of the bounds is greater than some threshold $\delta$.

     Since we can easily obtain inner DP results for all $(s,a)$ each time the DP is solved, we perform the upper and lower bound updates for all $(s,a)$ whenever an update is performed (as opposed to just at the current state-action pair). However, only the action-value of the current $(s,a)$ is projected between the lower and upper bounds, so the algorithm is still asynchronous.
     The pseudo-code of LBQL with experience replay with these changes, is shown in Algorithm \ref{A:ER-LBQL-SV}.

\begin{algorithm}[ht]
  \caption{LBQL with Experience Replay (Full Details)}
  \label{A:ER-LBQL-SV}
\begin{algorithmic}
  \STATE {\bfseries Input:}{Initial estimates $L_0 \le Q_0 \le U_0$, batch size $K$, stepsize rules $\alpha_n(s,a)$, $\beta_n(s,a),$} noise buffer $\mathcal B$ of size $\kappa$, number of steps between bound updates $m$, and threshold $\delta$.\\
  \STATE {\bfseries Output:}{Approximations $\{L_n\}$, $\{Q'_n\}$, and $\{U_n\}$.}
  \STATE Set $Q_0' = Q_0$ and choose an initial state $s_0$.\;
  \FOR{$n = 0, 1, 2, \ldots$}
  \STATE Choose an action $a_n$ via some behavior policy (e.g., $\epsilon$-greedy). Observe $w_{n+1}$.\;
  \STATE Store $w_{n+1}$ in $\mathcal B$.\;
  \STATE  Perform a standard $Q$-learning update:
  \vspace{-0.5em}
  \begin{equation*}
  \begin{aligned}
  Q_{n+1}(&s_n,a_n) = Q'_n(s_n,a_n) \\
  &+ \alpha_n(s_n, a_n) \Bigl[ r_n(s_n,a_n)+ \gamma \max_a Q'_n(s_{n+1},a) - Q'_n(s_{n},a_n) \Bigr ].
  \end{aligned}
\end{equation*}\;
  \vspace{-1.0em}
  \IF{$n \ge \kappa $ {\rm and} $n \,\, \mathbf{mod} \,\, m = 0$ {\rm and} $|U_n(s_n,a_n)-L(s_n,a_n)|>\delta$ }
  \vspace{0.1em}
  \STATE Sample randomly a batch $\mathcal D=\{w_1,w_2,\ldots, w_{K} \}$ and a sample path $\mathbf{w}=\{w_1,w_2,\ldots, w_{\tau} \}$ from $\mathcal B$, where $\tau \sim {\rm Geom}(1-\gamma)$.\;
  \STATE Set $\varphi = Q_{n+1}$. Using $\mathbf{w}$ and $\mathcal D$, compute $\hat  Q^{U}_{0}(s,a)$ and $\hat Q^L_{0}(s,a)$ for all $(s,a) \in \mathcal S \times \mathcal A$, using  \eqref{E:EmpQG} and \eqref{E:EmpQL}, respectively.\; 
  \STATE For all $(s,a) \in \mathcal S \times \mathcal A$, update upper and lower bounds:
  \vspace{-0.5em}
  \begin{align*}
 U_{n+1}(s_n,a_n) &= \Pi_{[-\rho, \, \infty]} \left [ U_n(s_n,a_n) + \beta_n(s_n,a_n) \, \bigl[ \hat Q^{U}_{0}(s_n,a_n) - U_n(s_n,a_n) \bigr] \right ],\nonumber\\
  L_{n+1}(s_n,a_n) &= \Pi_{[\infty, \, \rho]} \left [ L_n(s_n,a_n) + \beta_n(s_n,a_n) \, \bigl[\hat Q^L_{0}(s_n,a_n)- L_n(s_n,a_n) \bigr] \right ],
  \nonumber
  \end{align*}
  \ENDIF
  \STATE Enforce upper and lower bounds:
  \vspace{1em}
  \begin{equation*}
      Q'_{n+1}(s_n,a_n) = \Pi_{[ L_{n+1}(s_n,a_n),\, U_{n+1}(s_n,a_n)]} \left [Q_{n+1}(s_n,a_n)\right]
  \end{equation*}\nonumber
  
  \ENDFOR
  \end{algorithmic}
\end{algorithm}
\clearpage
\section{Numerical Experiments Details}
\label{Appendix:Experiments}
Let $\nu(s,a)$ and $\nu(s)$ be the number of times state-action pair $(s,a)$ and state $s$, have been visited, respectively.
For all algorithms, a polynomial learning rate $\alpha_n(s,a)= 1/\nu_n(s,a)^{r}$ is used, where $r=0.5$.
Polynomial learning rates have been shown to have a better performance than linear learning rates \citep{hasselt2010double}.

We use a discount factor of $\gamma=0.95$ for the pricing car-sharing/stormy gridworld problems, $\gamma=0.9$ for the windy gridworld problem and $\gamma=0.99$ for the repositioning problem.
Moreover, we use an $\epsilon$-greedy exploration strategy such that $\epsilon(s)=1/\nu(s)^{e}$, where $e$ is $0.4$ for the four-stations pricing car-sharing problem and $0.5$ for all the other problems. 
For the car-sharing/windy gridworld problems, the initial state-action values are chosen randomly such that $L_0(s,a) \le Q_0(s,a) \le U_0(s,a)$ where 
\[
L_0(s,a)= -R_{\max}/(1 - \gamma) \quad \text{and} \quad U_0(s,a)=R_{\max}/(1-\gamma)
\]
for all $(s,a)$. For the stormy gridworld problem, we set the initial state-action values to zero (we find that a random initialization caused all algorithms except LBQL to perform extremely poorly).

We report LBQL parameters used in our numerical experiments in Table \ref{Tab:LBQLparams}. Note that for a fair comparison, the parameter $K$ of bias-corrected Q-learning algorithm is taken equal to $K$ of LBQL in all experiments. In addition, the $\kappa$ steps used to create the buffer for LBQL are included in the total number of steps taken.
Results of the gridworld and car-sharing problems are averaged over 50 and 10 runs, respectively.
All experiments were run on a 
3.5 GHz Intel Xeon processor with 32 GB of RAM workstation.
\begin{table}[H]
  \centering
  \caption{LBQL parameters.}
    \begin{tabular}{lrrrrr}
    \toprule
      & \multicolumn{5}{c}{Parameter} \\
\cmidrule{2-6}    \multicolumn{1}{l}{Problem} & \multicolumn{1}{c}{$\beta$} & \multicolumn{1}{c}{$\kappa$} & \multicolumn{1}{c}{$K$} & \multicolumn{1}{c}{$m$} & \multicolumn{1}{c}{$\delta$} \\
    \midrule
    2-CS-R & 0.01 & 40 & 20 &  10 & 0.01 \\
    2-CS & 0.01 & 40 & 20 &  15 & 0.01 \\
    4-CS & 0.01 & 1000 & 20 &  200 & 0.01 \\
    WG & 0.2 & 100 & 10 & 10 & 0.01 \\
    SG & 0.2 & 500 & 20 & 20 & 0.05 \\
    \bottomrule
    \end{tabular}%
  \label{Tab:LBQLparams}%
\end{table}%
A detailed description of the environments is given in the next two sections.
\subsection{Gridworld Examples}
First we consider, windy gridworld, a well-known variant of the standard gridworld problem discussed in \cite{sutton2018reinforcement}. Then we introduce, stormy gridworld, a new environment that is more complicated than windy gridworld. The environments are summarized below.

    \textbf{Windy Gridworld.} 
    The environment is a $10\times 7$ gridworld, with a cross wind pointing \emph{upward}, \citep{sutton2018reinforcement}. The default \emph{wind} values corresponding to each  of the 10 columns are $\{0,0,0,1,1,1,2,2,1,0\}$. Allowable actions are \{up, right, down, left\}.
    If the agent happens to be in a column whose wind value is different from zero, the resulting next states are shifted upward by the ``wind'' whose intensity is stochastic, varying from the given values in each column by $\{ -1,0,1\}$ with equal probability. 
    Actions that corresponds to directions that takes the agent off the grid leave the location of the agent unchanged. The start and goal states are $(3,1)$ and $(3, 8)$, respectively. The reward is $-1$ until the goal state is reached, where the reward is 0 thereafter.

  \textbf{Stormy Gridworld.} 
    Consider the stochastic windy gridworld environment. 
    Now, however, we allow the wind to blow half the time as before and the other half it can blow from any of the three other directions. The horizontal wind values corresponding to each row from top to bottom are given by $\{0,0,1,1,1,1,0\}$.
    Also, it can randomly rain with equal probability in any of the central states that are more than two states away from the edges of the grid. The start and goal states are $(3,1)$ and $(3, 10)$ respectively. 
    Rain creates a puddle which affects the state itself and all of its neighboring states. 
    The reward is as before except when the agent enters a puddle state the reward is $-10$.

\subsection{Car-sharing Benchmark Examples}
In this section, we give the detailed formulations of the two variants of the car-sharing benchmark, repositioning and pricing. The essential difference is that in the pricing version, the decision maker ``repositions'' by setting prices to induce directional demand (but does not have full control since this demand is random).

\subsubsection{Repositioning Benchmark for Car-sharing}

We consider the problem of repositioning cars for a two stations car-sharing platform, \citep{he2019robust}. 
The action is the number of cars to be repositioned from one station to the other,  before random demand is realized. Since repositioning in both directions is never optimal, we use $r>0$ to denote the repositioned vehicles from station 1 to 2 and $r<0$ to denote repositioning from station 2 to 1.
The stochastic demands at time $t$ are $D_{1,t}$ and $D_{2,t}$ for stations 1 and 2 respectively, are i.i.d., discrete uniform, each supported on $ \{3,\ldots,9\}$. The rental prices are $p_1=3.5$ and $p_2=4$ for stations 1 and 2, respectively.
All rentals are one-way (i.e., rentals from station 1 end up at station 2, and vice-versa).
The goal is to maximize profit, where unmet demands are charged a lost sales cost $\rho_1=\rho_2=2$ and repositioning cost $c_1=1$ for cars reposition from station 1 to 2 and $c_2=1.5$ for cars repositioned from station 2 to 1. We assume a total of $\bar s=12$ cars in the system and formulate the problem as an MDP, with state $s_t \in \mathcal S=\{0,1,\ldots,12\}$ representing the number of cars at station 1 at beginning of period $t$. We denote by $V^*(s_t)$ the optimal value function starting from state $s_t$. The Bellman recursion is:
\begin{equation}
\begin{split}
V^*( s_{t}) &=  \max_{s_t-\bar s \le r_t \le s_t} \mathbf{E}  \bigg [ \sum_{i \in \{1,2\}} p_{i} \, \omega_{it}(D_{i,t+1})
- \sum_{i \in \{1,2\}}\rho_{i} \Big (D_{i,t+1}  - \omega_{it}(D_{i,t+1}) \Big ) \\ & \hspace{12em}- c_1 \max(r_t,0) + c_2 \min(r_t,0)
 + \gamma V^*({ s_{t+1})} \bigg],\\ &
 \omega_{1t}(D_{1,t+1}) = \min ( D_{1,t+1}, s_{t} - r_t  ),  \\ &
 \omega_{2t}(D_{2,t+1}) = \min ( D_{2,t+1}, \bar s-s_{t} + r_t   ),  \\ &
s_{t+1}=s_{t} - r_t + \omega_{2t}(D_{2,t+1}) - \omega_{1t}(D_{1,t+1}),
\end{split}
\label{MDP0}
\end{equation}
where $\gamma \in (0,1)$ is a discount factor. The repositioning problem for two stations is illustrated in Figure \ref{F:2-CS-R-Diagram}. The nodes represent stations, solid arcs represent fulfilled demands, and dashed arcs represent repositioned vehicles.


\subsubsection{Pricing Benchmark for Car-sharing}

Suppose that a vehicle sharing manager is responsible for setting the rental price for the vehicles at the beginning of each period in an infinite planning horizon.
We model a car sharing system with $N$ stations. The goal is to optimize the prices to set for renting a car at each of the $N$ stations; let the price at station $i$ and time $t$ be $p_{it} $ for $i \in [N] := \{1,2,\ldots, N\}$. 
Demands are nonnegative, independent and depends on the vehicle renting price according to a stochastic demand function
\begin{equation*}
D_{it}(p_{it},\epsilon_{i,t+1}):= \kappa_{i}(p_{it}) + \epsilon_{i,t+1},
\end{equation*}
where $D_{it}(p_{it},\epsilon_{i,t+1})$ is the demand in period $t$,
$\epsilon_{i,t+1}$ are random perturbations that are revealed at time $t+1$ and $\kappa_{i}(p_{it})$ is a deterministic demand function of the price $p_{it}$ that is set at the beginning of period $t$ at station $i \in [N]$. The random variables $\epsilon_{i,t+1}$ are independent with $\mathbf{E}[\epsilon_{i,t+1}]=0$ without loss of generality. 
Furthermore, we assume that the expected demand $\mathbf{E}[D_{it}(p_{it},\epsilon_{i,t+1})]=\kappa_{i}(p_{it})<\infty$ is strictly decreasing in the rental price $p_{it}$ which is restricted to a set of feasible price levels $ [\underline p_{i}, \overline p_{i}]$ for all $i \in [N]$, where $\underline  p_{i}, \overline p_{i}$ are the minimum and the maximum prices that can be set at station $i$, respectively. This assumption implies a one-to-one correspondence between the rental price $p_{it}$ and the expected demand $d_{it} \in \mathfrak{D}:=[\underline d_{i},\overline d_{i}]$ for all $p_{it} \in [\underline p_{i}, \overline p_{i} ]$ where $\underline d_{i}=\kappa_{i}(\overline p_{i})$  and $\overline d_{i}=\kappa_{i}(\underline p_{i})$.

The problem can be formulated as an MDP with state $\mathbf{s}_t$, which is a vector whose components represent the number of available cars at each of the $N$ stations at beginning of period $t$. The state space is $\mathcal{S}^{N-1}$ with $\mathcal S=\{0,1,\ldots,\bar{s} \}$ and $\bar{s}$ is the maximum number of cars in the vehicle sharing system.
We assume that a customer at station $i$ goes to station $j$ with probability $\phi_{ij} $ for all $i,j \in [N]$.
Let $Y_{ik,t+1}$ be a random variable taking values in $[N]$ that represents the random destination of customer $k$ at station $i$, which is only observed at the beginning of period $t+1$. We have $Y_{ik,t+1}= j$ with probability  $\phi_{ij}$, so $Y_{ik,t+1}$ are i.i.d. for each customer $k$.
Denote by $l_{ij}$ the distance from station $i$ to $j$, for all $i,j \in [N]$.
We penalize unmet demands by a lost sales unit cost $\rho_{i}$, $i \in [N]$.
The decision vector is $\mathbf{p}_t = \{p_{it}  \, \in [\underline  p_{i}, \overline p_{i}],   \forall i \in [N] \bigr \}$.
Let $V^*(\mathbf{s}_t)$ be the revenue-to-go function with number of available vehicles $\mathbf{s}_t$. Thus,  we have the Bellman recursion
\begin{equation}
\begin{split}
V^*(\mathbf{s}_t) &=  \max_{\mathbf{p}_t} \; \mathbf{E}  \bigg [ \sum_{i \in [N]} p_{it} \, \sum_{j \in [N]} l_{ij} \, \omega_{ijt}(\epsilon_{i,t+1})
- \sum_{i \in [N]}\rho_{i} \Big (\kappa_{i}(p_{it}) + \epsilon_{i,t+1} - \omega_{it}(\epsilon_{i,t+1}) \Big )
 + \gamma V^*({ \mathbf{s}_{t+1})} \bigg]\\ &
 \omega_{it}(\epsilon_{i,t+1}) = {\rm min} \, ( \kappa_{i}(p_{it}) + \epsilon_{i,t+1}, s_{it}   )  \quad \forall \, i \in [N],  \\ &
 \omega_{ijt}(\epsilon_{i,t+1})=\sum_{k=1}^{\omega_{it}(\epsilon_{i,t+1})} \mathbbm{1}_{\{Y_{ik,t+1}=j\}} \quad \forall \, i,j \in [N], \\ &
s_{i,t+1}=s_{it} + \sum_{j \in [N]} \omega_{jit}(\epsilon_{i,t+1}) - \omega_{it}(\epsilon_{i,t+1}), \quad \forall \, i \in [N],
\end{split}
\label{MDP1}
\end{equation}
where $\gamma \in (0,1)$ is a discount factor.
Note that the MDP in \eqref{MDP1} can be reformulated using the action-value function $Q( \mathbf{s}_{t},{p_t})$ instead of $V(\mathbf{s}_{t})$.
The quantity $\omega_{it}(\epsilon_{i,t+1})$ represents the total fulfilled customer trips from station $i$ at time $t$ for a given realization of the noise $\epsilon_{i,t+1}$.
Notice that in \eqref{MDP1} there are two sources of randomness: the noise due to stochastic demand represented by $\epsilon_i,$ for all $i \in [N]$ and the noise due to the random distribution of fulfilled rentals between the stations, i.e., due to the random variables $Y_{i1},\ldots, Y_{i\omega_{it}(\epsilon_{i,t+1})}$ for all $i \in [N]$.
Due to the high dimensionality involved in the state, action, and noise spaces, solving \eqref{MDP1} is computationally challenging. 

\paragraph{Spatial Pricing in Two-Location Car-sharing.}
We first consider the pricing problem on two stations and $12$ cars in total.  
The state space is $\mathcal S= \{0,1,\ldots,12\}$ representing the number of cars at station $1$.
All rentals are one-way.
The prices, at each period $t$, are restricted to $p_{1t} \in [1, 6]$ and $p_{2t} \in [1, 7]$.
The stochastic demand functions at period $t$ are given by: $ D_{1t}(p_{1t},\epsilon_{1,t+1}):= 9 - p_{1t} + \epsilon_{1,t+1}$ and $ D_{2t}(p_{2t},\epsilon_{2,t+1}):= 10 - p_{2t} + \epsilon_{2,t+1}$  for stations 1 and 2 respectively. The random variables $\epsilon_{1,t+1}$ and $\epsilon_{2,t+1}$ are independent, discrete uniform, each supported on $ \{-3,-2,\ldots,3\}$.
We use the discretized expected demands, as our actions: $d_{1t} \in \{3,\ldots,8\}$ and $d_{2t} \in \{3,\ldots,9\}$. The lost sales cost is $2$ at both stations.

\paragraph{Spatial Pricing in Four-Location Car-sharing.}
Consider the car-sharing problem for four stations with $\bar s=20$ cars and $d_{it} \in \{3,4\}$ for each station. In total there are $1771$ states and $16$ actions.  
The random variables $\epsilon_{i,t+1}$ are independent, discrete uniform, each supported on $ \{-3,-2,\ldots,3\}$. We consider both one way and return trips at each station. Figure \ref{F:CSP4S} shows an illustration of the stations (nodes) and the rentals between the stations (arcs). The probabilities $\phi_{ij}=0.25$ for all $i,j \in \{1,2,3,4 \}$ and the lost sales costs $(\rho_i)$ are $1.7, 1.2, 1.5, 2$ at stations $1,2,3,4$, respectively. The distance between the stations are taken such that $l_{ij}=1$ if $i=j$, and the other distances being symmetrical, meaning  $l_{ij}=l_{ji}$ with $l_{12}=1.8$,  $l_{13}=1.5$, $l_{14}=1.4$,  $l_{23}=1.6$,  $l_{24}=1.1$, and  $l_{34}=1.2$.

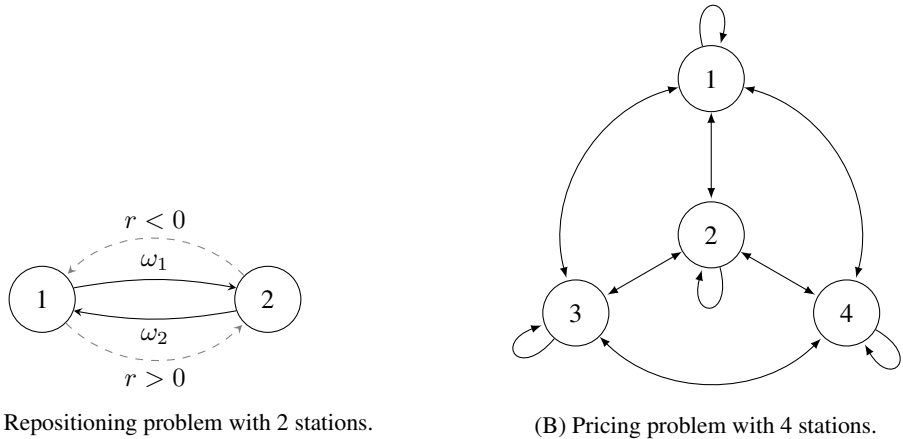
\begin{figure*}[htpb] 
\centering
	\begin{subfigure}{0.4\textwidth}
	\vspace{2.75cm}
	\hspace{1.2cm}
	\begin{tikzpicture}[ ->,>= stealth, shorten >=1pt,node distance=3cm,on grid,auto]
	\node[state] (1) {1};
	\node[state] (2) [right=of 1] {2};
	\path[->]
	(1)
	edge [color=gray,->,swap,dashed,bend right=45] node [color=black] {$r>0$} (2)
	(2)
    edge [color=gray,->,dashed,swap,bend right=45] node  [color=black] {$r<0$} (1)
	(1.20)
	edge [shorten >=0pt,bend left=10] node  {$\omega_1$} (2.160)
	(2.200)
	edge [shorten >=0pt,bend left=10]  node {$\omega_2$} (1.340)
	;
	\end{tikzpicture}
	\caption{Repositioning problem with 2 stations.}
	\label{F:2-CS-R-Diagram}
	\end{subfigure}
	\hspace{0.1cm}
	\begin{subfigure}{0.4\textwidth}
	\hspace{0.5cm}
	\begin{tikzpicture}[ ->,>= latex, shorten >=1pt,node distance=1.8cm,on grid,auto]
	\node[state] (3) at (0,0) {3};
	\node[state] (4) at (3.6,0) {4};
	\node[state] (1) at (1.8, 3.1177) {1};
	\node[state] (2) at (1.8,1.0392) {2};
	\path[<->]
	(1) 
	edge [loop above] ()
	(2) 
	edge [loop below] ()
	(1) 
	edge node {} (2)
	(2)
	edge node {} (3)	
	(3)
	edge [bend right=45] node {} (4)
	(2)
	edge node {} (4)
	(1)
	edge  [bend right=45] node {} (3)
	(1)
	edge  [bend left=45] node {} (4);
	\draw (4) to [out=330,in=300,looseness=8] (4);
	\draw (3) to [out=230,in=200,looseness=8] (3);
	\end{tikzpicture}
	\caption{Pricing problem with 4 stations.}
	\label{F:CSP4S}
	\end{subfigure}
\caption{Illustrations of the repositioning and pricing car-sharing problems.} \label{F:cs_figures}
\end{figure*}
\clearpage
\subsection{Sensitivity Analysis}
\label{Appendix:SA}
We also perform sensitivity analysis on the five algorithms with respect to the learning rate and exploration parameters $r$ and $e$ for the car-sharing problem with two stations. Here, $r$ controls the polynomial learning rate defined by, $\alpha_n(s,a)= 1/\nu_n(s,a)^{r}$ and $e$ controls the $\epsilon$-greedy exploration strategy, where $\epsilon$ is annealed according to $\epsilon(s)=1/\nu(s)^{e}$. 
We use $\nu(s,a)$ and $\nu(s)$ to denote the number of times a state-action pair $(s,a)$ and state $s$, have been visited, respectively. We report our results in Table \ref{tab:longtablenew1}.
These results show the average number of iterations and CPU time until each algorithm first reach 50\%, 20\%, 5\%, 1\% relative error for each case of the parameters $e$ and $r$ while keeping all other parameters as before.
The ``-'' indicates that the corresponding \% relative error for the corresponding case was not achieved during the course of training.
The values in the table are obtained by averaging five independent runs for each case.
Except for the few cases where BCQL performs slightly better, LBQL once again drastically outperforms the other algorithms and exhibits \emph{robustness} against the learning rate and exploration parameters, an important practical property that the other algorithms seem to lack.

The effect of varying parameters $m$ and $K$ of LBQL is presented in Figure \ref{F:CS_sensitivity}. 
These plots are obtained by tuning parameters $m\in\{1, 10,50,150,200\}$ and $K \in \{1,5,10,100,1000 \}$ of LBQL algorithm in the car-sharing problem with two stations.
All other parameters are kept the same as before. Figures \ref{F:CS_Performance_M_sensitivity} and \ref{F:CS_Performance_K_sensitivity} show the mean total reward with a $95\%$ CI. Figures \ref{F:CS_RE_M_sensitivity} and \ref{F:CS_RE_K_sensitivity} show the mean and $95\%$ CI of the relative error given by: $\|V_n- V^*\|_2/\| V^* \|_2$. The results are obtained from 10 independent runs. Using larger values of $m$ reduces the strength of LBQL in both the performance and relative error metrics, as shown in Figures \ref{F:CS_Performance_M_sensitivity} and \ref{F:CS_RE_M_sensitivity}. This is expected since the effect of the bounds fades as we update the bounds less frequently. 
Interestingly, we can see from the performance plot that $m=10$ strikes a good balance between how often to do the bounds and Q-learning updates and achieves a performance that is slightly better and more stable than that of $m=1$ after about half of the training process (50,000 steps).
In terms of the sample size $K$, Figures \ref{F:CS_Performance_K_sensitivity} and \ref{F:CS_RE_K_sensitivity} clearly show that larger values of $K$ improve the performance of LBQL in terms of performance and relative error measures. This is not unexpected because a larger sample yields a better approximation of the penalty.

\begin{figure}[htbb] 
\centering
	\begin{subfigure}{0.4\textwidth}
		\includegraphics[width=\linewidth]{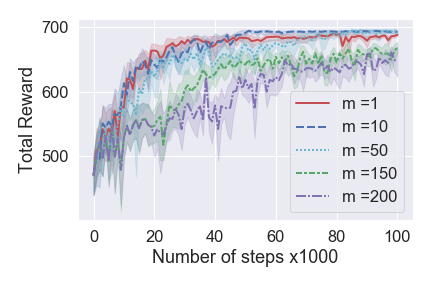}
		\vspace{-10pt}
		\caption{Performance (2-CS)} \label{F:CS_Performance_M_sensitivity}
	\end{subfigure}	
	\begin{subfigure}{0.4\textwidth}
		\includegraphics[width=\linewidth]{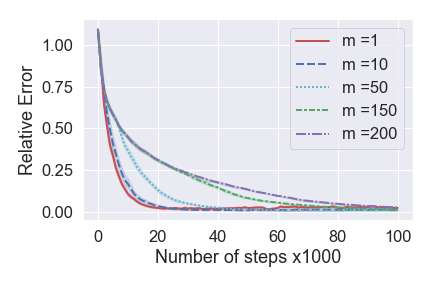}
		\vspace{-10pt}
		\caption{Relative Error (2-CS)} \label{F:CS_RE_M_sensitivity}
	\end{subfigure}
		\begin{subfigure}{0.4\textwidth}
		\includegraphics[width=\linewidth]{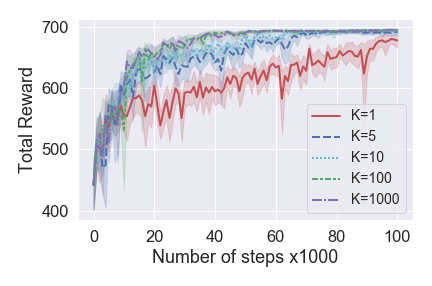}
		\caption{Performance (2-CS)} \label{F:CS_Performance_K_sensitivity}
	\end{subfigure}	
	\begin{subfigure}{0.4\textwidth}
		\includegraphics[width=\linewidth]{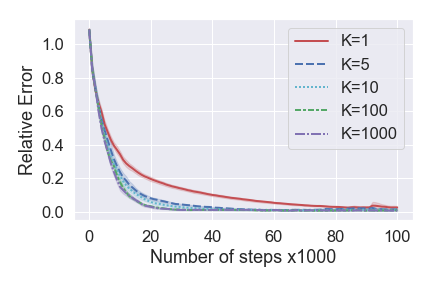}
		\vspace{-10pt}
		\caption{Relative Error (2-CS)} \label{F:CS_RE_K_sensitivity}
	\end{subfigure}
	\caption{Plots showing the effect of tuning the parameter $M$ and $K$ of LBQL algorithm.
}
	\label{F:CS_sensitivity}
\end{figure}

\footnotesize
\clearpage
\vspace{-1em}
\label{Appendix:Tables}
\begingroup
\renewcommand\arraystretch{1.2}
\begin{longtable}{crrrrrrrrrr}
\caption{Computational results for different exploration \& learning rate parameters. Bold numbers indicate the best performing algorithm.}\\
\toprule         
&       &       & \multicolumn{8}{c}{\% Relative error} \\
          & \multicolumn{1}{l}{$e$} & \multicolumn{1}{l}{$r$} & \multicolumn{2}{c}{50\%} & \multicolumn{2}{c}{20\%} & \multicolumn{2}{c}{5\%} & \multicolumn{2}{c}{1\%} \\
\cmidrule{4-11}          &       &       & \multicolumn{1}{c}{$n$} & \multicolumn{1}{c}{$t$ (s)} & \multicolumn{1}{c}{$n$} & \multicolumn{1}{c}{$t$ (s)} & \multicolumn{1}{c}{$n$} & \multicolumn{1}{c}{$t$ (s)} & \multicolumn{1}{c}{$n$} & \multicolumn{1}{c}{$t$ (s)} \\
\cmidrule{2-11}   
\endfirsthead
\caption[]{\textit{(continued)}}\\
\toprule         
&       &       & \multicolumn{8}{c}{\% Relative error} \\
          & \multicolumn{1}{l}{$e$} & \multicolumn{1}{l}{$r$} & \multicolumn{2}{c}{50\%} & \multicolumn{2}{c}{20\%} & \multicolumn{2}{c}{5\%} & \multicolumn{2}{c}{1\%} \\
\cmidrule{4-11}          &       &       & \multicolumn{1}{c}{$n$} & \multicolumn{1}{c}{$t$ (s)} & \multicolumn{1}{c}{$n$} & \multicolumn{1}{c}{$t$ (s)} & \multicolumn{1}{c}{$n$} & \multicolumn{1}{c}{$t$ (s)} & \multicolumn{1}{c}{$n$} & \multicolumn{1}{c}{$t$ (s)} \\
\cmidrule{2-11} 
\endhead
\endfoot
\multicolumn{11}{r}{\emph{\small (continued on next page)}} \\
\endfoot

\endlastfoot
 \multirow{15}[2]{*}{{\rotatebox[origin=c]{90}{\texttt{LBQL}}}} & 0.4 & 0.5 &          3,672.6  & 1.9 & \textbf{9,323.6 } & \textbf{4.6} & \textbf{18,456.6} & \textbf{8.9} & \textbf{33,054.0} & \textbf{15.7} \\
  &   & 0.6 & \textbf{3,632.0} & 1.7 & \textbf{9,147.4} & \textbf{4.4} & \textbf{18,270.0} & \textbf{8.6} & \textbf{39,624.8} & \textbf{18.6} \\
  &   & 0.7 & \textbf{3,725.2} & \textbf{1.8} & \textbf{9,087.4} & \textbf{4.3} & \textbf{18,217.8} & \textbf{8.6} & \textbf{41,941.2} & \textbf{19.7} \\
  &   & 0.8 & \textbf{3,698.2} & \textbf{1.8} & \textbf{9,321.8} & \textbf{4.4} & \textbf{20,860.0} & \textbf{9.9} & \textbf{53,752.6} & \textbf{25.6} \\
  &   & 0.9 & \textbf{3,992.2} & \textbf{1.9} & \textbf{10,119.2} & \textbf{4.9} & \textbf{23,070.2} & \textbf{11.0} & \textbf{80,252.8} & \textbf{38.4} \\
  & 0.5 & 0.5 &          3,316.0  & 1.6 & \textbf{8,040.2} & \textbf{3.8} & \textbf{15,050.2} & \textbf{7.2} & \textbf{27,912.8} & \textbf{13.2} \\
  &   & 0.6 &          3,514.4  & 1.7 & \textbf{8,529.4} & \textbf{4.0} & \textbf{16,595.6} & \textbf{7.9} & \textbf{36,100.2} & \textbf{17.0} \\
  &   & 0.7 & \textbf{3,531.8} & \textbf{1.7} & \textbf{8,712.8} & \textbf{4.1} & \textbf{17,835.6} & \textbf{8.6} & \textbf{46,010.0} & \textbf{21.9} \\
  &   & 0.8 & \textbf{3,449.2} & \textbf{1.6} & \textbf{8,571.8} & \textbf{4.0} & \textbf{18,152.6} & \textbf{8.5} & \textbf{65,007.6} & \textbf{30.4} \\
  &   & 0.9 & \textbf{3,398.4} & \textbf{1.6} & \textbf{8,346.4} & \textbf{3.9} & \textbf{18,844.8} & \textbf{8.8} & \textbf{99,820.2} & \textbf{46.6} \\
  & 0.6 & 0.5 &          2,877.4  & 1.3 & \textbf{7,129.0} & \textbf{3.3} & \textbf{13,046.0} & \textbf{6.2} & \textbf{23,822.0} & \textbf{11.2} \\
  &   & 0.6 &          3,182.8  & 1.5 & \textbf{8,066.0} & \textbf{3.8} & \textbf{15,421.4} & \textbf{7.3} & \textbf{33,286.0} & \textbf{15.6} \\
  &   & 0.7 & \textbf{2,979.4} & 1.4 & \textbf{7,625.6} & \textbf{3.6} & \textbf{15,414.6} & \textbf{7.2} & \textbf{34,238.0} & \textbf{15.9} \\
  &   & 0.8 & \textbf{3,272.6} & \textbf{1.6} & \textbf{8,431.0} & \textbf{4.1} & \textbf{17,809.2} & \textbf{8.5} & \textbf{     114,032.8} & \textbf{54.2} \\
  &   & 0.9 & \textbf{3,185.6} & \textbf{1.5} & \textbf{8,480.0} & \textbf{4.1} & \textbf{19,242.4} & \textbf{9.2} & \textbf{     123,331.2} & \textbf{58.8} \\
\cmidrule{2-11}    \multirow{16}[2]{*}{{\rotatebox[origin=c]{90}{\texttt{BCQL}}}}  & 0.4 & 0.5 & \textbf{3,200.8} & \textbf{1.0} &        22,455.0  & 7.0 &        65,329.0  & 20.3 &      107,785.6  & 33.7 \\
  &   & 0.6 &          4,618.2  & \textbf{1.5} &        43,724.6  & 13.6 &      159,662.6  & 49.6 &      292,421.0  & 34.9 \\
  &   & 0.7 &          8,059.4  & 2.5 &      123,484.2  & 38.4 &  -  & - &  -  & - \\
  &   & 0.8 &        17,287.0  & 5.3 &  -  & - &  -  & - &  -  & - \\
  &   & 0.9 &        67,162.2  & 20.9 &  -  & - &  -  & - &  -  & - \\
  & 0.5 & 0.5 & \textbf{2,209.6} & \textbf{0.7} &        15,604.0  & 4.9 &        48,715.6  & 15.3 &        80,317.4  & 25.2 \\
  &   & 0.6 & \textbf{3,274.2} & \textbf{1.0} &        31,422.8  & 9.8 &      124,319.6  & 38.7 &      243,101.4  & 75.6 \\
  &   & 0.7 &          5,619.6  & 1.8 &        89,857.0  & 27.8 &  -  & - &  -  & - \\
  &   & 0.8 &        11,417.0  & 3.6 &  -  & - &  -  & - &  -  & - \\
  &   & 0.9 &        42,605.4  & 13.1 &  -  & - &  -  & - &  -  & - \\
  & 0.6 & 0.5 & \textbf{1,830.4} & \textbf{0.6} &        11,639.6  & 3.6 &        35,763.0  & 11.1 &        61,249.0  & 19.0 \\
  &   & 0.6 & \textbf{2,612.4} & \textbf{0.8} &        23,571.6  & 7.4 &        92,101.4  & 28.8 &      177,127.6  & 55.5 \\
  &   & 0.7 &          4,371.2  & \textbf{1.3} &        66,526.0  & 20.5 &  -  & - &  -  & - \\
  &   & 0.8 &          9,028.2  & 2.8 &      297,368.6  & 17.9 &  -  & - &  -  & - \\
  &   & 0.9 &        31,673.6  & 9.8 &  -  & - &  -  & - &  -  & - \\
\cmidrule{2-11}    \multirow{15}[2]{*}{{\rotatebox[origin=c]{90}{\texttt{SQL}}}} & 0.4 & 0.5 &          7,750.6  & 1.9 &        37,889.8  & 9.1 &        93,820.0  & 22.5 &      141,171.0  & 33.8 \\
  &   & 0.6 &        11,329.4  & 2.8 &        75,364.0  & 18.3 &      233,422.0  & 56.4 &  -  & - \\
  &   & 0.7 &        20,131.4  & 4.8 &      212,767.0  & 51.0 &  -  & - &  -  & - \\
  &   & 0.8 &        46,986.8  & 11.3 &  -  & - &  -  & - &  -  & - \\
  &   & 0.9 &      182,890.0  & 43.8 &  -  & - &  -  & - &  -  & - \\
  & 0.5 & 0.5 &          6,122.2  & 1.5 &        30,944.8  & 7.4 &        79,167.4  & 19.0 &      120,527.8  & 29.0 \\
  &   & 0.6 &          9,166.6  & 2.2 &        62,540.6  & 14.9 &      201,822.4  & 48.2 &  -  & - \\
  &   & 0.7 &        15,835.6  & 3.8 &      174,233.6  & 42.0 &  -  & - &  -  & - \\
  &   & 0.8 &        36,548.8  & 8.7 &  -  & - &  -  & - &  -  & - \\
  &   & 0.9 &      157,029.0  & 37.7 &  -  & - &  -  & - &  -  & - \\
  & 0.6 & 0.5 &          4,984.0  & 1.2 &        24,989.0  & 6.0 &        64,605.4  & 15.4 &        98,554.6  & 23.6 \\
  &   & 0.6 &          7,396.2  & 1.8 &        50,282.4  & 12.0 &      165,574.2  & 39.8 &  -  & - \\
  &   & 0.7 &        13,018.8  & 3.1 &      143,142.6  & 34.1 &  -  & - &  -  & - \\
  &   & 0.8 &        29,201.0  & 6.9 &  -  & - &  -  & - &  -  & - \\
  &   & 0.9 &      122,335.6  & 29.2 &  -  & - &  -  & - &  -  & - \\[0.3cm]
 \cmidrule{2-11}    \multirow{15}[2]{*}{{\rotatebox[origin=c]{90}{\texttt{QL}}}}  & 0.4 & 0.5 &          7,743.0  & 1.7 &        38,114.2  & 8.2 &        93,303.4  & 20.2 &      136,851.4  & 29.6 \\
  &   & 0.6 &        11,644.0  & 2.5 &        76,625.0  & 16.6 &      232,679.4  & 50.9 &  -  & - \\
  &   & 0.7 &        20,181.6  & 4.4 &      212,401.4  & 46.3 &  -  & - &  -  & - \\
  &   & 0.8 &        45,987.2  & 10.1 &  -  & - &  -  & - &  -  & - \\
  &   & 0.9 &      191,442.2  & 41.9 &  -  & - &  -  & - &  -  & - \\
  & 0.5 & 0.5 &          6,143.6  & 1.3 &        30,996.2  & 6.8 &        78,131.8  & 16.9 &      116,361.2  & 25.3 \\
  &   & 0.6 &          9,331.6  & 2.0 &        63,998.2  & 13.9 &      204,593.6  & 44.6 &  -  & - \\
  &   & 0.7 &        16,247.0  & 3.5 &      178,842.8  & 38.4 &  -  & - &  -  & - \\
  &   & 0.8 &        38,297.0  & 8.2 &  -  & - &  -  & - &  -  & - \\
  &   & 0.9 &      165,835.8  & 35.7 &  -  & - &  -  & - &  -  & - \\
  & 0.6 & 0.5 &          5,005.2  & 1.1 &        24,877.2  & 5.4 &        63,777.8  & 13.7 &        96,402.0  & 20.8 \\
  &   & 0.6 &          7,547.0  & 1.9 &        51,369.4  & 13.1 &      166,179.2  & 42.3 &      289,882.8  & 46.1 \\
  &   & 0.7 &        13,288.2  & 3.1 &      144,318.2  & 33.1 &  -  & - &  -  & - \\
  &   & 0.8 &        30,172.6  & 6.5 &  -  & - &  -  & - &  -  & - \\
  &   & 0.9 &      139,952.6  & 30.3 &  -  & - &  -  & - &  -  & - \\  
\cmidrule{2-11}  
\multirow{15}[2]{*}{{\rotatebox[origin=c]{90}{\texttt{Double-QL}}}} & 0.4   & 0.5   &      224,490.2  & 51.2  &  -  & -     &  -  & -     &  -  & - \\
          &       & 0.6   &  -  &         - &  -  & -     &  -  & -     &  -  & - \\
          &       & 0.7   &  -  &         - &  -  & -     &  -  & -     &  -  & - \\
          &       & 0.8   &  -  &         - &  -  & -     &  -  & -     &  -  & - \\
          &       & 0.9   &  -  &         - &  -  & -     &  -  & -     &  -  & - \\
          & 0.5   & 0.5   &  -  &         - &  -  & -     &  -  & -     &  -  & - \\
          &       & 0.6   &  -  &         - &  -  & -     &  -  & -     &  -  & - \\
          &       & 0.7   &  -  &         - &  -  & -     &  -  & -     &  -  & - \\
          &       & 0.8   &  -  &         - &  -  & -     &  -  & -     &  -  & - \\
          &       & 0.9   &  -  &         - &  -  & -     &  -  & -     &  -  & - \\
          & 0.6   & 0.5   &  -  &         - &  -  & -     &  -  & -     &  -  & - \\
          &       & 0.6   &  -  &         - &  -  & -     &  -  & -     &  -  & - \\
          &       & 0.7   &  -  &         - &  -  & -     &  -  & -     &  -  & - \\
          &       & 0.8   &  -  &         - &  -  & -     &  -  & -     &  -  & - \\
          &       & 0.9   &  -  &         - &  -  & -     &  -  & -     &  -  & - \\
            &   &   &   &   &   &   &   &   &   &  \\
\bottomrule   
\label{tab:longtablenew1}
\end{longtable}
\endgroup